\documentclass{article}

\usepackage{microtype}
\usepackage{graphicx}
\usepackage{subfigure}
\usepackage{booktabs} 

\usepackage{hyperref}

\usepackage[accepted]{icml2023}

\usepackage{amsmath}
\usepackage{amssymb}
\usepackage{mathtools}
\usepackage{amsthm}

\usepackage[capitalize,noabbrev]{cleveref}

\usepackage{algorithm}
\usepackage{algorithmic}

\usepackage{caption}
\usepackage{svg}
\usepackage{minitoc}

\usepackage{enumitem} 

\definecolor{darkblue}{rgb}{0.0, 0.0, 0.8}
\definecolor{darkred}{rgb}{0.8, 0.0, 0.0}
\definecolor{darkgreen}{rgb}{0.0, 0.8, 0.0}
\definecolor{purple}{RGB}{153,50,204}

\newcommand*{\prob}{\mathcal{P}}
\newcommand{\cpl}{\mathcal{C}}

\newcommand{\dW}{d_\mathrm{W}}

\newcommand{\dWLk}{d_{\mathrm{WL}}^{\scriptscriptstyle{(k)}}}
\newcommand{\dGk}{d_{\mathcal{G},q}^{\scriptscriptstyle{(k)}}}
\newcommand{\dGek}{d_{\mathcal{G},(\eps)}^{\scriptscriptstyle{(k)}}}

\newcommand{\eps}{\varepsilon}

\newcommand{\R}{\mathbb{R}}

\newcommand{\lc}{\left(}
\newcommand{\rc}{\right)}

\newcommand*{\mX}{\mathcal{X}}
\newcommand*{\mY}{\mathcal{Y}}

\newcommand{\N}{\mathbb{N}}
\newcommand{\NN}{\mathcal{N\!N}}

\newcommand{\WLh}[2]{\mathfrak{l}^{\scriptscriptstyle{(#1)}}_{\scriptscriptstyle{#2}}}

\newcommand{\PP}{\mathbb{P}}
\newcommand{\law}{\mathrm{law}}

\theoremstyle{plain}
\newtheorem{theorem}{Theorem}[section]
\newtheorem{proposition}[theorem]{Proposition}
\newtheorem{lemma}[theorem]{Lemma}
\newtheorem{corollary}[theorem]{Corollary}
\theoremstyle{definition}
\newtheorem{definition}[theorem]{Definition}

\theoremstyle{remark}
\newtheorem{remark}[theorem]{Remark}
\newtheorem{claim}{Claim}

\usepackage[textsize=tiny]{todonotes}

\icmltitlerunning{The Weisfeiler-Lehman Distance: Reinterpretation and Connection with GNNs}

\begin{document}

\twocolumn[
\icmltitle{The Weisfeiler-Lehman Distance: Reinterpretation and Connection with GNNs}



\icmlsetsymbol{equal}{*}
\begin{icmlauthorlist}
    \icmlauthor{Samantha Chen}{equal,to}
    \icmlauthor{Sunhyuk Lim}{equal,goo}
    \icmlauthor{Facundo Mémoli}{equal,a}
    \icmlauthor{Zhengchao Wan}{equal,b}
    \icmlauthor{Yusu Wang}{equal,to,b}
    \end{icmlauthorlist}
    
    \icmlaffiliation{to}{Department of Computer Science and Engineering, University of California San Diego, La Jolla, California, USA}
    \icmlaffiliation{goo}{Max Planck Institute for Mathematics in the Sciences, Leipzig, Saxony, Germany}
    \icmlaffiliation{a}{Department of Mathematics and Department of Computer Science and Engineering, The Ohio State University, Columbus, Ohio, USA}
    \icmlaffiliation{b}{Hal{\i}c{\i}o\u{g}lu Data Science Institute, University of California San Diego, La Jolla, California, USA}
    
    \icmlcorrespondingauthor{Zhengchao Wan}{zcwan@ucsd.edu}

\icmlkeywords{Machine Learning, ICML}

\vskip 0.3in
]



\printAffiliationsAndNotice{\icmlEqualContribution} 

\begin{abstract}
In this paper, we present a novel interpretation of the Weisfeiler-Lehman (WL) distance introduced by \cite{chen2022weisfeilerlehman} using concepts from stochastic processes. The WL distance compares graphs with node features, has the same discriminative power as the classic Weisfeiler-Lehman graph isomorphism test and has deep connections to the Gromov-Wasserstein distance. Our interpretation connects the WL distance to the literature on distances for stochastic processes, which also makes the interpretation of the distance more accessible and intuitive. We further explore the connections between the WL distance and certain Message Passing Neural Networks, and discuss the implications of the WL distance for understanding the Lipschitz property and the universal approximation results for these networks.
\end{abstract}

\section{Introduction}
\label{section:introduction}

The Weisfeiler-Lehman (WL) test, a classic graph isomorphism test \citep{leman1968reduction} which has recently gained renewed interest as a tool for analyzing Message Passing Graph Neural Networks (MP-GNNs) \citep{xu2018powerful,azizian2020expressive}. 
Recently, \citet{chen2022weisfeilerlehman} introduced the Weisfeiler-Lehman (WL) distance between labeled measure Markov chains (LMMCs). 
The WL distance has the same power as the WL test in distinguishing non-isomorphic graphs and it is more discrimative than a certain WL based graph kernel \cite{togninalli2019wasserstein}. Moreover, \citet{chen2022weisfeilerlehman} unveiled interesting connections between the WL distance and both a certain neural network architecture on Markov chains and a variant of the Gromov-Wasserstein distance \citep{memoli2011gromov}. 

Although the WL distance possesses nice theoretical properties and good empirical performance in graph classification tasks, the original formulation of the WL distance is complicated and can be hard to decipher. 
In this work, we identify a novel characterization of the WL distance using concepts from stochastic processes. This new characterization eventually provides an alternative, more intuitive reformulation of the WL distance. Via this reformulation, we further identify connections between the WL distance and the so-called Causal Optimal Transport (COT), a branch of the Optimal Transport theory specifically taylored for comparing stochastic processes \citep{lassalle2018causal}.

Finally, we recall that
\citet{chen2022weisfeilerlehman} introduced a certain neural network structure, called Markov Chain Neural Networks (MCNNs), for Markov chains and utilized the WL distance to explain theoretical properties of MCNNs. 
It was mentioned that MCNNs will reduce to a special type of MP-GNNs when restricted to Markov chains induced by graphs. In this work, we proceed further along this line and explicitly clarify how the WL distance can be used to understand properties of MP-GNNs. In particular, inspired by the analysis of \cite{chen2022weisfeilerlehman,chuang2022tree}, we establish the Lipschitz property and a universal approximation result for such MP-GNNs. 

{\bf Related work.~}
Our work builds on recent developments in graph similarity measures, particularly the WL distance introduced by \cite{chen2022weisfeilerlehman}. Another relevant work is the Tree Mover's (TM) distance, introduced in \citep{chuang2022tree}, which compares labeled graphs using Wasserstein distance and has similar discriminative power to the WL distance.  
However, in contrast to the combinatorial computation tree structure used in the TM distance (which currently cannot yet handle weighted graphs), the WL distance benefits from a more flexible Markov chain formulation. This formulation enables the WL distance to compare weighted graphs, potentially handle continuous objects, such as heat kernels on Riemannian manifolds, and facilitates the development of differentiable distances for comparing labeled graphs, as demonstrated in \cite{brugere2023distances}. 
Additionally, \citet{toth2022capturing} introduced the hypo-elliptic graph Laplacian and a corresponding diffusion model that captures the evolution of random walks on graphs. This approach shares similarities with the WL test and the WL distance by focusing on entire walk trajectories instead of individual steps. Exploring the relationship between these concepts in future research would be intriguing.

\section{Preliminaries}\label{sec:pre}
\subsection{Probability measures and Optimal Transport}\label{sec:probability}
Let $Z$ be any metric space. We let $\mathcal{P}(Z)$ denotes the space of all Borel probability measures on $Z$ with finite 1-moment\footnote{This is equivalent to saying that for any $\alpha\in\mathcal{P}(Z)$ and any $z_0\in Z$, one has that $\int_Zd_Z(z,z_0)\alpha(dz)<\infty$.}. We let $\mathcal{P}^{\circ 1}(Z):=\mathcal{P}(Z)$ and for each $k=1,\ldots$, we inductively let $\mathcal{P}^{\circ (k+1)}(Z):=\mathcal{P}(\mathcal{P}^{\circ k}(Z))$.

Given any measurable map $f:X\rightarrow Y$ and a probability measure $\alpha\in\mathcal{P}(X)$, we let $f_\#\alpha\in\mathcal{P}(Y)$ denote the \emph{pushforward} of $\alpha$, i.e., for any measurable $B\subseteq Y$, 
$f_\#\alpha(B):=\alpha(f^{-1}B).$

For any probability measures $\alpha,\beta$ on $Z$, the ($\ell_1$) Wasserstein distance between them is defined as follows:
\[\dW(\alpha,\beta):=\inf_{\pi\in\mathcal{C}(\alpha,\beta)}\int_{Z\times Z} d_Z(z,z')\pi(dz\times dz').\]
Here $\mathcal{C}(\alpha,\beta)$ denotes the set of all couplings between $\alpha$ and $\beta$, i.e., the set of all probability measure $\pi\in\mathcal{P}(Z\times Z)$ such that marginals of $\pi$ are $\alpha$ and $\beta$, respectively.

\paragraph{A note on notation for probability measures.} We will mostly deal with finite sets in this paper. Given a finite set $X$ and $\alpha\in\mathcal{P}(X)$, we let $\alpha(x):=\alpha(\{x\})$ for any $x\in X$.

\paragraph{An alternative description of the Wasserstein distance.} A random variable $X$ with values in a complete and separable metric space $Z$ is any measurable map from a probability space $(\Omega,\mathbb{P})$ to $Z$. We let $\mathrm{law}(X)\coloneqq X_\#\mathbb{P}$ denote the law of $X$.
Now, we define the notion of {coupling} in terms of random variables.
\begin{definition}[Coupling in terms of random variables]\label{def: coupling rv}
Given two probability measures $\alpha,\beta$ on a complete and separable metric space $Z$, we call any pair of random variables $\lc{X},{Y}\rc:(\Omega,\mathbb{P})\rightarrow Z\times Z$ a coupling between $\alpha$ and $\beta$ if $\mathrm{law}(X)=\alpha$ and $\law(Y)=\beta$.
\end{definition}
Of course, given a coupling $(X,Y)$, one has that $\law((X,Y))\in\mathcal{C}(\alpha,\beta)$. Conversely, note that given any (measure-theoretical) coupling $\pi\in\mathcal{C}(\alpha,\beta) $, one can always find a coupling $\lc{X},{Y}\rc$ between $\alpha$ and $\beta$ such that $\mathrm{law}\lc\lc{X},{Y}\rc\rc=\pi$. 

Now, given two probability measures $\alpha$ and $\beta$ on a metric space $Z$, the Wasserstein distance between $\alpha$ and $\beta$ can be rewritten using the language of random variables as follows:
\begin{equation}\label{eq:expected dw}
\dW(\alpha,\beta)=\inf_{\lc X,Y\rc}\mathbb{E}\, d_Z\lc{X},{Y}\rc,
\end{equation}
where the infimum is taken over all couplings $\lc X,Y\rc$ between $\alpha$ and $\beta$. 
\subsection{Markov chains}
Let $X$ be a finite set. We denote by $m_\bullet^X:X\rightarrow \mathcal{P}(X)$ a Markov transition kernel on $X$. Let $\mu_X\in\mathcal{P}(X)$ be a \emph{stationary} distribution w.r.t. $m_\bullet^X$. Then, we call the tuple $\mathcal{X}:=(X,m_\bullet^X,\mu_X)$ a measure Markov chain (MMC). 

Due to the Kolmogorov extension theorem \citep{kolmogorov2018foundations} (see also \citep[Theorem 2.1.14]{durrett2019probability}), an equivalent way of describing a measure Markov chain $\mathcal{X}:=(X,m_\bullet^X,\mu_X)$ is to view it as a probability measure $\mathbb{P}_X$ on the path space $X^\N=X\times X\times \cdots$, i.e., $X^\N=\{w=(x_i)_{i=0}^\infty:\,x_i\in X\}$\footnote{Here $\N$ denotes the set of all non negative integers.}: If we let $X_i:X^\N\rightarrow X$ denote the projection map to the $i$-th component for $i\in\N$, then $\mathbb{P}_X$ is required to satisfy that 
\begin{itemize}
    \item $\mathbb{P}_X(X_{i+1}(w)=x'|X_i(w)=x)=m_x^X(x')$ for any $x,x'\in X$ and
    \item $(X_0)_\#\mathbb{P}_X=\mu_X$.
\end{itemize}
As $X_i$ can be viewed as a random variable on the probability space $(X^\N,\PP_X)$ with values in $X$, the second condition can be also rewritten as $\law(X_0)=\mu_X$.

Now given any metric space $Z$, consider any map $\ell_X:X\rightarrow Z$. Then, we call the tuple $(\mathcal{X},\ell_X)$ a $Z$-labeled measure Markov chain (($Z-$)LMMC).

\subsection{The Weisfeiler-Lehman Distance}\label{sec:WL distance}

In \citep{chen2022weisfeilerlehman}, a notion of (pseudo-)distance is proposed for labeled measure Markov chains (LMMCs). The motivation comes from the classical Weisfeiler-Lehman (WL) graph isomorphism test, which compares two graphs by iteratively testing whether certain aggregated node-label summaries of the two input graphs are the same. The WL distance of \citep{chen2022weisfeilerlehman} introduced a measure-theoretic treatment of the node labels via Markov kernels, and it essentially ``metrized'' the WL-test procedure into a distance measure that is compatible with WL test. 
We briefly introduce the concept below; see the original paper \citep{chen2022weisfeilerlehman} for details. 

Consider a $Z-$LMMC $(\mathcal{X},\ell_X)$. We recursively define a sequence of maps $\WLh{k}{(\mX,\ell_X)}:X\rightarrow \mathcal{P}^{\circ k}(Z)$ for $k\in \N$. First of all, we let $\WLh{0}{(\mX,\ell_X)}:=\ell_X$. Then,
$$\WLh{k+1}{(\mX,\ell_X)}\coloneqq\lc\WLh{k}{(\mX,\ell_X)}\rc_\# m_\bullet^X:X\rightarrow \mathcal{P}^{\circ k}(Z).$$ 
Finally, we let 
$$\mathfrak{L}_k\!\lc(\mX,\ell_X)\rc\coloneqq \lc \WLh{k}{(\mX,\ell_X)}\rc_\#\mu_X\in\prob^{\circ(k+1)}(Z).$$

Now, given two $Z-$LMMCs $(\mathcal{X},\ell_X)$ and $(\mathcal{Y},\ell_Y)$ and any $k\in\N$, the Weisfeiler-Lehman distance of depth $k$ is defined as follows:
\begin{equation}\label{eq:definition of dwlk}
    \dWLk\!\lc(\mX,\ell_X),(\mY,\ell_Y)\rc\coloneqq\dW\!\lc \mathfrak{L}_k\!\lc(\mX,\ell_X)\rc,\mathfrak{L}_k\!\lc(\mY,\ell_Y)\rc\rc.
\end{equation}

\section{Reinterpretation of the WL distance}\label{sec:sec4}
The original definition of the WL distance, while well-suited for devising a computation algorithm for the distance, is rather intricate due to its iterative consideration of probability measures on spaces $\mathcal{P}^{\circ k}(Z)$ with increasing complexity. However, a better understanding of the concept can be achieved by using the language of stochastic processes. In this section, we will elaborate on two approaches: (1) using the concept of \emph{Markovian couplings}, a special type of couplings between Markov chains, and (2) exploring the connection between the WL distance and the theory of causal Optimal Transport \citep{lassalle2018causal,backhoff2017causal}, a variant of OT specialized for comparing stochastic processes. These interpretations of the WL distance offer valuable insights for future research in this field.

\subsection{Recall of a characterization of the WL distance}
We first recall a characterization of the WL distance from \citep{chen2022weisfeilerlehman}. 
Given any two MMCs $\mX$ and $\mY$, one can inductively define the notion of $k$-step coupling between $m_\bullet^X$ and $m_\bullet^Y$ as follows:
\begin{itemize}
    \item[$k=1$:] A \emph{1-step coupling} between $m_\bullet^X$ and $m_\bullet^Y$ is defined to be any measurable map 
    \[\nu^{\scriptscriptstyle{(1)}}_{\bullet,\bullet}:X\times Y\rightarrow \mathcal{P}(X\times Y)\]
    such that $\nu^{\scriptscriptstyle{(1)}}_{x,y}\in\mathcal{C}(m_x^X,m_y^Y) $ for any $x\in X$ and $y\in Y$.
    \item[$k\geq 2$:] A measurable map 
    $\nu^{\scriptscriptstyle{(k)}}_{\bullet,\bullet}:X\times Y\rightarrow \mathcal{P}(X\times Y)$ is called a \emph{$k$-step coupling} between $m_\bullet^X$ and $m_\bullet^Y$ if there exist a $(k-1)$-step coupling $\nu^{(k-1)}_{\bullet,\bullet}$ and a $1$-step coupling
    $\nu^{\scriptscriptstyle{(1)}}_{\bullet,\bullet}$ such that for any $x\in X$ and $y\in Y$, one has
    \[\nu^{\scriptscriptstyle{(k)}}_{x,y}=\int_{X\times Y}\nu^{(k-1)}_{x',y'}\cdot\nu^{\scriptscriptstyle{(1)}}_{x,y}(dx'\times dy')\]
    
\end{itemize}

Let $\mathcal{C}^{\scriptscriptstyle{(k)}}(m_\bullet^X,m_\bullet^Y)$ denote the collection of all $k$-step couplings between $m_\bullet^X$ and $m_\bullet^Y$. 
Furthermore, for any $\gamma\in\mathcal{C}(\mu_X,\mu_Y)$ and any $\nu^{\scriptscriptstyle{(k)}}_{\bullet,\bullet}\in\mathcal{C}^{\scriptscriptstyle{(k)}}(m_\bullet^X,m_\bullet^Y)$, denote by 
\[\nu_{\bullet,\bullet}^{\scriptscriptstyle{(k)}}\odot\gamma:= \int_{X\times Y}\nu^{\scriptscriptstyle{(k)}}_{x',y'}\cdot\gamma(dx'\times dy').\]
It is shown in \citep{chen2022weisfeilerlehman} that $\nu_{\bullet,\bullet}^{\scriptscriptstyle{(k)}}\odot\gamma\in\mathcal{C}(\mu_X,\mu_Y)$. 
We then let $\mathcal{C}^{\scriptscriptstyle{(k)}}(\mu_X,\mu_Y)$ denote the collection of all such couplings. 
Then, it turns out the WL distance can be characterized as follows\footnote{In Theorem A.7 of \citep{chen2022weisfeilerlehman}, it is assumed that the measures equipped in LMMCs are stationary with respect to their corresponding Markov kernels. However, we note that the statement in Theorem A.7 still holds if we remove the requirement measures to be stationary.}.

\begin{theorem}[{\citep[Theorem A.7]{chen2022weisfeilerlehman}}]\label{thm:dwl= dwk}
    Given $k\in\N$ and any two $Z$-LMMCs $(\mX,\ell_X)$ and $(\mY,\ell_Y)$, one has that
    $$\dWLk((\mX,\ell_X),(\mY,\ell_Y))=\!\!\!\!\!\!\!\!\inf_{\gamma^{\scriptscriptstyle{(k)}}\in\mathcal{C}^{\scriptscriptstyle{(k)}}(\mu_X,\mu_Y)}\!\!\!\!\!\!\!\!\mathbb{E}_{\gamma^{\scriptscriptstyle{(k)}}}d_Z(\ell_X(X_0),\ell_Y(Y_0)).$$
\end{theorem}

\subsection{Markovian couplings and a stochastic process interpretation of the WL distance}\label{sec:markovian coupling}
Readers familiar with stochastic processes may recognize the construction of $k$-step couplings in the previous section. 
In fact, we observe that these couplings can essentially be derived from the so-called Markovian couplings (which we will introduce later).
Based on this observation, in this section, we provide a clean and intuitive interpretation of the WL distance as a variant of the Wasserstein distance between distributions of paths by highlighting this connection.

Recall that any measure Markov chain $(X,m_\bullet^X,\mu_X)$ can be equivalently described as a probability measure $\mathbb{P}_X$ on the path space $X^\N$ such that $(X_0)_\#\PP_X=\mu_X$ and $\mathbb{P}_X(X_{i+1}=x'|X_i=x)=m_x^X(x')$ for any $x,x'\in X$. 
In this way, given any two MMCs $\mathcal{X}=(X,m_\bullet^X,\mu_X)$ and $\mathcal{Y}=(Y,m_\bullet^Y,\mu_X)$, one can consider couplings $\PP\in\mathcal{C}(\PP_X,\PP_Y)\subseteq \mathcal{P}(X^\N\times Y^\N)$ between their corresponding path distributions $\PP_X$ and $\PP_Y$, respectively.
There is a canonical identification between $X^\N\times Y^\N$ and $(X\times Y)^\N$ sending $((x_j)_{j\in\N},(y_j)_{j\in\N})$ to $(x_j,y_j)_{j\in\N}$. Then, $\PP$ can be also considered as a probability measure on $(X\times Y)^\N$, i.e., a distribution of paths inside the space $X\times Y$.
As this distribution is a coupling between two  Markov chains, it is natural to also impose the Markov property on such couplings. 
This naturally leads to the following definition of Markovian couplings which helps to provide an intuitive characterization of the WL distance.
We first introduce some notation. For any $i\in\N$, we let $X_i:X^\N\times Y^\N\rightarrow X$ sending $((x_j)_{j\in\N},(y_j)_{j\in\N})$ to $x_i$ and we similarly define $Y_i$.
Now, we introduce the notion of Markovian couplings.

\begin{definition}[Markovian couplings]\label{def: markov coupling}
A coupling $\PP\in\mathcal{C}(\PP_X,\PP_Y)$ is called a \emph{Markovian coupling}, if the following conditions are satisfied: 
\begin{enumerate}[leftmargin=*]
    \item The sequence $\{(X_i,Y_i)\}_{i\in \N}$, regarded as random variables valued in $X\times Y$ on the probability measure space $(X^\N\times Y^\N,\PP)$, satisfies the \emph{Markov property}, i.e., for any $i\in \N$, $x_0,\ldots,x_{i}\in X$ and $y_0,\ldots,y_{i}\in Y$, one has that
    \begin{align*}
&\mathbb{P}\big((X_i,Y_i)=(x_i,y_i)|\{(X_j,Y_j)=(x_j,y_j)\}_{0\leq j\leq i-1}\big)\\
&=\mathbb{P}\big((X_i,Y_i)=(x_i,y_i)|(X_{i-1},Y_{i-1})=(x_{i-1},y_{i-1})\big).
    \end{align*}
    \item For any $i\in\N$, $x\in X$ and $y\in Y$, if we construct a probability measure $(\nu_{i+1})_{x,y}\in \mathcal{P}(X\times Y)$ as follows
\begin{align*}
    &(\nu_{i+1})_{x,y}(x',y'):=\\
    &\mathbb{P}\big((X_{i+1},Y_{i+1})=(x',y')|(X_{i},Y_{i})=(x,y)\big)
\end{align*}
$\forall$ $x'\in X$, $y'\in Y$, then $(\nu_{i+1})_{x,y}\in\mathcal{C}(m_x^X,m_y^Y)$. i.e., $(\nu_{i+1})_{\bullet,\bullet}$ is a 1-step coupling between $m_\bullet^X$ and $m_\bullet^Y$.\footnote{When the event $\{(X_i,Y_i)=(x,y)\}$ is null, the conditional probability is not defined and hence we simply let $(\nu_{i+1})_{x,y}$ be the product measure $m_x^X\otimes m_y^Y\in\mathcal{C}(m_x^X,m_y^Y)$.}
\end{enumerate}
We let $\mathcal{C}_\mathcal{M}(\PP_X,\PP_Y)$ denote the collection of all Markovian couplings between $\PP_X$ and $\PP_Y$.
\end{definition}

In short, a Markovian coupling of two Markov chains $\PP_X$ and $\PP_Y$ is a time-inhomogeneous Markov chain on a state space $X\times Y$ that is the product of the state spaces of the two chains and whose transition kernel at each time step is a coupling between the transition kernels of the two chains; see \Cref{fig:markovian coupling} for an illustration.

\begin{remark}[Initial distribution]
   Note that coupling $\PP\in\mathcal{C}(\PP_X,\PP_Y)$ satisfies that $(X_0,Y_0)_\#\PP\in\mathcal{C}(\mu_X,\mu_Y)$. Hence, any Markovian coupling also starts with an initial distribution which is itself a coupling between the initial distributions of the two Markov chains.
\end{remark}

\begin{figure}[htb!]
    \centering
    \includegraphics[width=0.2\textwidth]{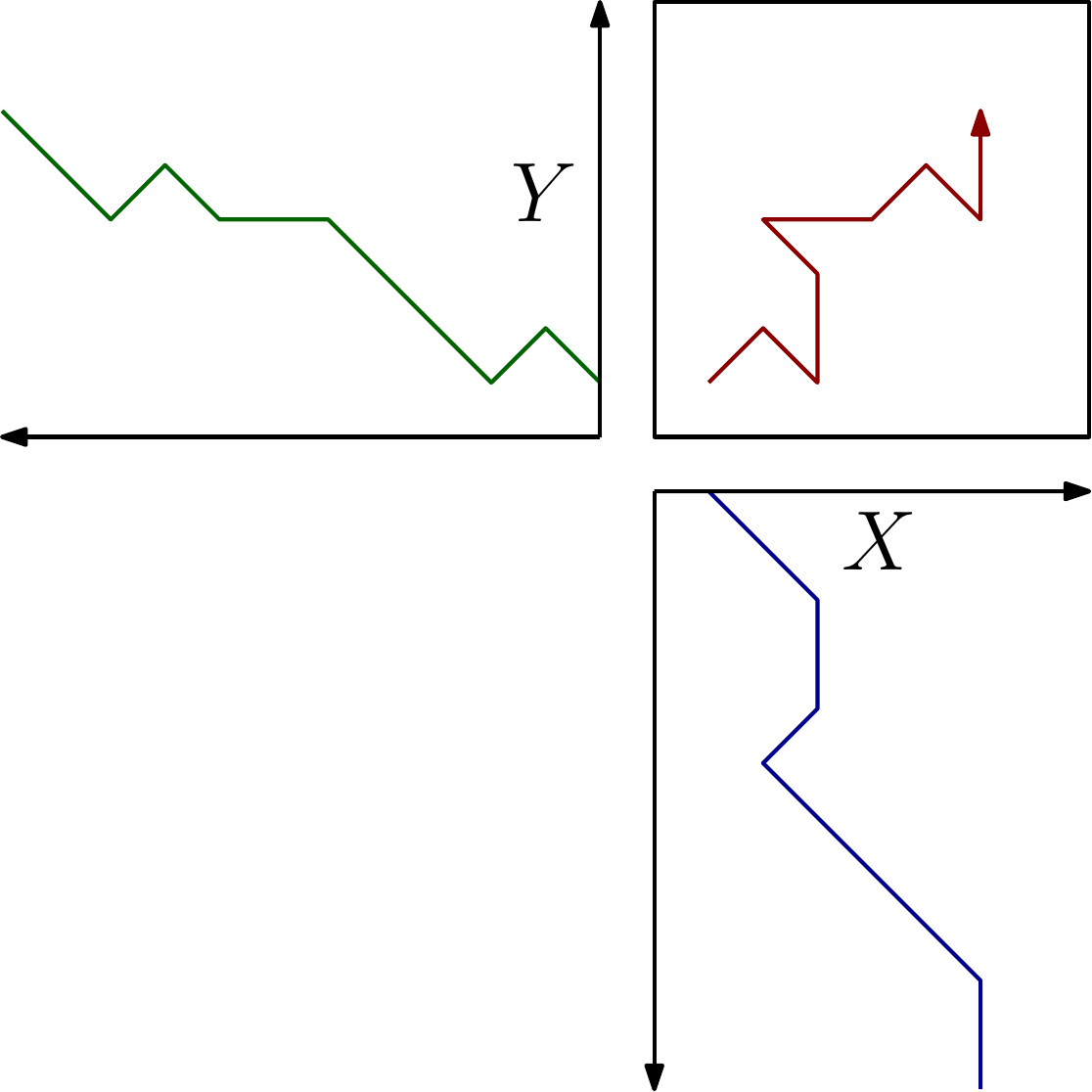}
    \caption{A Markovian coupling between $\mX$ and $\mY$ is a Markov chain in the product space $X\times Y$ so that its projections on to $X$ and $Y$ have the same law as $\mX$ and $\mY$, respectively.}
    \label{fig:markovian coupling}
\end{figure}

\begin{remark}[Coupling method]
    This definition of Markovian coupling is a generalization of a common coupling technique used for proving convergence of Markov chains (see for example \citep{levin2017markov}). 
    Note that whereas we define Markovian couplings for two arbitrary Markov chains, in the context of ``coupling technique'', the two Markov chains involved have the same transition kernels.
    \end{remark}

\begin{remark}[Existence of Markovian couplings]\label{rmk:MarkCoupExist}
Given transition kernels $(\nu_{i+1})_{\bullet,\bullet}$ for all steps $i=0,\ldots$ and an initial distribution $\gamma$, one can always construct a Markovian coupling via the Kolmogorov extension theorem \citep{kolmogorov2018foundations}. 
\end{remark}

As promised at the beginning of this section, below we provide an intuitive explanation of $k$-step couplings. See \Cref{sec:proof of prop:mcoupling=kcoupling} for the proof of this result. 

\begin{proposition}\label{prop:mcoupling=kcoupling}
For any coupling $\gamma\in\mathcal{C}(\mu_X,\mu_Y)$ and any $k$-step coupling $\nu^{\scriptscriptstyle{(k)}}_{\bullet,\bullet}\in\mathcal{C}^{\scriptscriptstyle{(k)}}(m_\bullet^X,m_\bullet^Y)$, there exists a Markovian coupling $\PP\in\mathcal{C}_\mathcal{M}(\PP_X,\PP_Y)$ such that 
$$\mathrm{law}((X_k,Y_k))=\nu_{\bullet,\bullet}^{\scriptscriptstyle{(k)}}\odot\gamma\,\,\text{and}\,\,\mathrm{law}((X_0,Y_0))=\gamma.$$
Conversely, for any Markovian coupling $\PP\in\mathcal{C}_\mathcal{M}(\PP_X,\PP_Y)$, $\gamma:=\mathrm{law}((X_0,Y_0))$ is a coupling between $\mu_X$ and $\mu_Y$ and furthermore, there exists $k$-step coupling $\nu_{\bullet,\bullet}^{\scriptscriptstyle{(k)}}\in\cpl^{\scriptscriptstyle{(k)}}\!\lc m_\bullet^{X},m_\bullet^Y\rc$ such that $\mathrm{law}((X_k,Y_k))=\nu_{\bullet,\bullet}^{\scriptscriptstyle{(k)}}\odot\gamma$.
\end{proposition}

This proposition implies that any $k$-step coupling, just as indicated by the name, is the distribution of two coupled random walks at exactly step $k$.
The direct implication of this result is that we provide a characterization of the WL distance of depth $k$ as follows. We note that the following theorem is an immediate consequence of \Cref{thm:dwl= dwk} and \Cref{prop:mcoupling=kcoupling}.

\begin{theorem}\label{thm:dwlstochastic}
For any $Z$-LMMCs $(\mX,\ell_X)$ and $(\mY,\ell_Y)$, we have that
$$\dWLk\lc(\mX,\ell_X),(\mY,\ell_Y)\rc=\inf_{\PP\in\mathcal{C}_\mathcal{M}(\PP_X,\PP_Y)}\mathbb{E}_{\PP}\, d_Z(\ell_X({X}_k),\ell_Y({Y}_k)).$$
\end{theorem}

In this way, we have interpreted the WL distance of depth $k$ as a variant of the Wasserstein distance between the distributions of paths corresponding to the input Markov chains. Here the cost between two paths $w_X\in X^\N$ and $w_Y\in Y^\N$ is given by $d_Z(\ell_X({X}_k(w_X),\ell_Y({Y}_k(w_Y)))$ through the label functions. Essentially, the WL distance at depth $k$ is comparing the distribution of the Markov chains at step $k$; see \Cref{fig:dWL} for an illustration.

\begin{figure}[htb!]
    \centering
\includegraphics[width=0.3\textwidth]{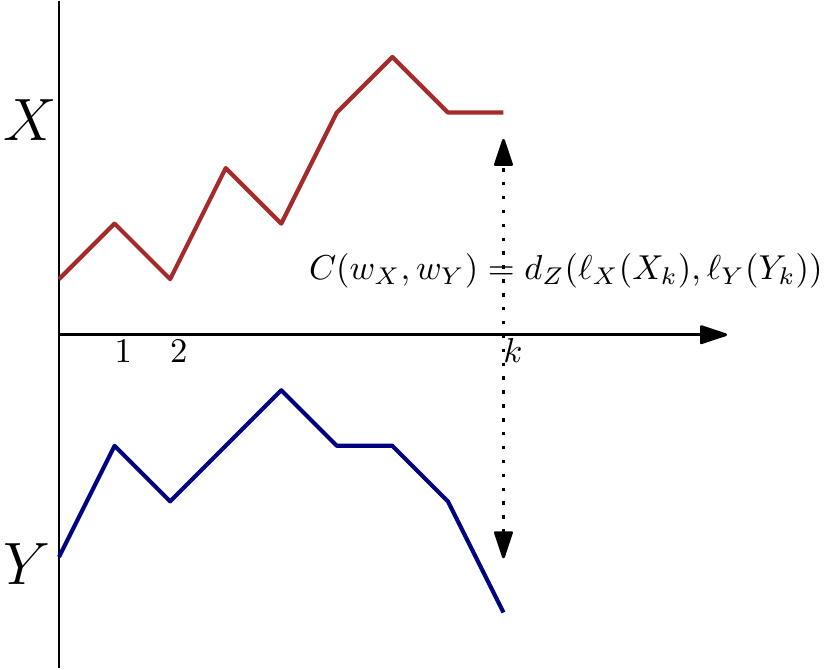}
    \caption{The WL distance of depth $k$ is almost like the Wasserstein distance between path distributions where the cost $C(w_X,w_Y)$ between paths is given by the distance between labels of the two random walks at step $k$: $d_Z(\ell_X({X}_k(w_X),\ell_Y({Y}_k(w_Y)))$.}
    \label{fig:dWL}
\end{figure}

\subsubsection{In terms of the language of random variables}
A Markov chain can of course be realized on probability spaces beyond the path space.
This flexibility can be useful in situations where some random variables need to be independent of the Markov chain.
We consider this slightly more generalized setting in this section and reformulate the definition of Markovian couplings and also \Cref{thm:dwlstochastic}. It is worth noting that skipping this section will not impede understanding of the other parts of the paper.

Note that a common way for describing a measure Markov chain $\mathcal{X}=(X,m_\bullet^X,\mu_X)$ is through a sequence of random variables $\{X_i:\Omega\rightarrow X\}_{i\in \N}$ from some probability space $(\Omega,\mathbb{P})$ such that $\mathrm{law}(X_0)=\mu_X$ and $\mathbb{P}(X_{i+1}=x'|X_i=x)=m_x^X(x')$ for any $x,x'\in X$.
Just to distinguish from our terminology of measure Markov chains $\mathcal{X}$, we call any such sequence of random variables as a \emph{stochastic realization} of the MMC $\mathcal{X}$.
In this way, the projection maps $\{X_i:(X^\N,\PP_X)\rightarrow X\}_{i\in\N}$ on the path space give a concrete example of stochastic realization. 

We now generalize the Markovian couplings from \Cref{def: markov coupling} to the case of general stochastic realizations in a way similar to \Cref{def: coupling rv}.
\begin{definition}[Markovian couplings for stochastic realizations]
Given MMCs $\mathcal{X}=(X,m_\bullet^X,\mu_X)$ and $\mathcal{Y}=(Y,m_\bullet^Y,\mu_Y)$, let $\{X_i\}_{i\in \N}$ and $\{Y_i\}_{i\in \N}$ be stochastic realizations on the \emph{same} probability space $(\Omega,\mathbb{P})$. We call the sequence of random variables $\{(X_i,Y_i)\}_{i\in \N}$ a \emph{Markovian coupling} between $\mathcal{X}$ and $\mathcal{Y}$, if the following properties hold: 

\begin{enumerate}[leftmargin=*]
\item The sequence of random variables $\{(X_i,Y_i)\}_{i\in \N}$ satisfies the Markov property, i.e., for any $i\in \N$, $x_0,\ldots,x_{i}\in X$ and $y_0,\ldots,y_{i}\in Y$, one has that
    \begin{align*}
&\mathbb{P}\big((X_i,Y_i)=(x_i,y_i)|\{(X_j,Y_j)=(x_j,y_j)\}_{0\leq j\leq i-1}\big)\\
&=\mathbb{P}\big((X_i,Y_i)=(x_i,y_i)|(X_{i-1},Y_{i-1})=(x_{i-1},y_{i-1})\big).
    \end{align*}
    \item For any $i\in\N$, $x\in X$ and $y\in Y$, if we construct a probability measure in $\mathcal{P}(X\times Y)$ as follows
\begin{align*}
    &(\nu_{i+1})_{x,y}(x',y'):=\\
    &\mathbb{P}\big((X_{i+1},Y_{i+1})=(x',y')|(X_{i},Y_{i})=(x,y)\big),
\end{align*}

then $(\nu_{i+1})_{x,y}\in\mathcal{C}(m_x^X,m_y^Y)$, i.e., $(\nu_{i+1})_{\bullet,\bullet}$ is a 1-step coupling between $m_\bullet^X$ and $m_\bullet^Y$.
\end{enumerate}
\end{definition}
Note that the above definition is almost identical to the definition of Markovian couplings between path distribution (see \Cref{def: markov coupling}).
Hence, it is not surprising that \Cref{thm:dwlstochastic} can be rephrased in terms of stochastic realizations; see also the similarity between \Cref{eq:expected dw} and the corollary below:

\begin{corollary}\label{thm:dwlstochastic-expectation}
For any $Z$-LMMCs $(\mX,\ell_X)$ and $(\mY,\ell_Y)$, we have that
\[\dWLk\lc(\mX,\ell_X),(\mY,\ell_Y)\rc=\inf_{\left\{\lc{X}_i,{Y}_i\rc\right\}_{i\in \N}}\mathbb{E}\, d_Z(\ell_X({X}_k),\ell_Y({Y}_k))\]
where the infimum is taken over all possible Markovian couplings between $\mathcal{X}$ and $\mathcal{Y}$.
\end{corollary}

\subsection{The Connection with Causal Optimal Transport}\label{sec:COT}

Finally, we comment that the Markovian coupling characterization formula given in the previous section is deeply connected with the notion of Causal Optimal Transport (COT) \citep{lassalle2018causal}. We elucidate this point in this section.

Note that COT has already found applications in mathematical finance \citep{glanzer2019incorporating,backhoff2020adapted} and in machine learning \citep{xu2020cot,xu2022CCOT,klemmer2022spate}. 
As COT lies in the framework of Optimal Transport, multiple Sinkhorn algorithms have been proposed in the literature \citep{pichler2022nested,eckstein2022computational} for accelerating computations. The connection between the WL distance and COT that we study in this section may eventually result in efficient algorithms for computing/approximating the WL distance. We leave this for future study. 
Furthermore, we believe that this connection can also be useful for extending the WL distance to the case when input LMMCs have different labeling space $Z$'s.

We first review some basics of COT. Given two (finite) spaces $X$ and $Y$ and an integer $k\in\N$, consider the product spaces $X^{k+1}$ and $Y^{k+1}$ (these product spaces are viewed as spaces of paths in $X$ or $Y$ of length $k+1$)\footnote{The framework of COT can incorporate the setting of both discrete and continuous paths of finite or infinite length. In this paper, we will focus on paths of finite length for simplicity and clarity of our presentation.}. Let $\alpha\in\mathcal{P}(X^{k+1})$ and $\beta\in \mathcal{P}(Y^{k+1})$. Now we are ready to define the notion of (bi)causal coupling between $\alpha$ and $\beta$.

\begin{definition}[(bi)causal coupling]
A coupling measure $\pi\in\cpl(\alpha,\beta)$ is said to be \emph{causal} from $\alpha$ to $\beta$ if it satisfies
$$\pi((y_0,\dots,y_l)\vert (x_0,\dots,x_k))=\pi((y_0,\dots,y_l)\vert (x_0,\dots,x_l))$$
for all $l\in\{0,\dots,k\}$ and $(x_0,\dots,x_k)\in X^{k+1}$. 
Here, the notataion $\pi(\cdot|\cdot)$ denotes conditional probability. This implies that, at time $l$ and given the past $(x_0,\dots,x_l)$ of $X$, the distribution of $y_l$ does not depend on the future $(x_{l+1},\dots,x_k)$ of $X$.

Moreover, $\pi$ is said to be \emph{bicausal} if it is causal both from $\alpha$ to $\beta$ and from $\beta$ to $\alpha$. The set of bicausal couplings between $\alpha$ and $\beta$ will be denoted by $\cpl_{\mathrm{bc}}(\alpha,\beta)$.
\end{definition}

Now, let $c:X^{k+1}\times Y^{k+1}\rightarrow \R_+$ be any cost function. Then, the \emph{bicausal OT problem} is formulated as follows:
\[d^c(\alpha,\beta):=\inf_{\pi\in\cpl_{\mathrm{bc}}(\alpha,\beta)}\int_{X^{k+1}\times Y^{k+1}}c(x,y)\pi(dx\times dy).\]
Similarly, one can formulate the causal OT (COT) problem by considering causal couplings. We use the acronym COT to refer to both causal and bicausal OT problems.

It turns out that Markovian couplings naturally give rise to bicausal couplings by restricting paths with infinite time steps from $\N$ to paths with finite time steps.
More precisely, given two MMCs $\mathcal{X}=(X,m_\bullet^X,\mu_X)$ and $\mathcal{Y}=(Y,m_\bullet^Y,\mu_X)$, we consider any Markovian coupling $\PP\in\mathcal{C}_\mathcal{M}(\PP_X,\PP_Y)$ {(cf. \Cref{def: markov coupling})}. 
Then, for any $k\in\N$, we let $\alpha^k:=\mathrm{law}((X_0,\ldots,X_k))\in\mathcal{P}(X^{k+1})$ and let $\beta^k:=\mathrm{law}((Y_0,\ldots,Y_k))\in\mathcal{P}(Y^{k+1})$.
Recall that here $\law((X_0,\ldots,X_k)):=(X_0,\ldots,X_k)_\#\PP$ and $\mathrm{law}((Y_0,\ldots,Y_k))$ is similarly defined. 
In fact, it is easy to see that $\alpha^k$ and $\beta^k$ can be expressed explicitly as follows. 
\begin{align*}
    \alpha^k((x_0,x_1,\ldots x_k))&= m_{x_{k-1}}^X(x_k)\cdots m_{x_0}^{X}(x_1)\mu_X(x_0),\\
    \beta^k((y_0,y_1,\ldots y_k)) &= m_{y_{k-1}}^Y(y_k)\cdots m_{y_0}^{Y}(y_1)\mu_Y(y_0).
\end{align*}
Now, we further let 
\[\pi^k:=\mathrm{law}((X_0,\ldots,X_k,Y_0,\ldots,Y_k))\in\mathcal{P}(X^{k+1}\times Y^{k+1}).\]
Then, we have that
\begin{lemma}\label{lemma:MCtoBC}
$\pi^k$ is a bicausal coupling between $\alpha^k$ and $\beta^k$.
\end{lemma}

Now, we consider the following cost function $c^k:X^{k+1}\times Y^{k+1}\rightarrow \R_+$ defined by 
\begin{equation}\label{ck}
   ((x_0,\ldots,x_k),(y_0,\ldots,y_k))\mapsto d_Z(\ell_X(x_k),\ell_Y(y_k)) 
\end{equation}
and state the following theorem.

\begin{theorem}\label{thm:dwl cot}
For any $Z$-LMMCs $(\mX,\ell_X)$ and $(\mY,\ell_Y)$, we have that
\[\dWLk\lc(\mX,\ell_X),(\mY,\ell_Y)\rc=d^{c^k}(\alpha^k,\beta^k).\]
\end{theorem}

In order to prove Theorem \ref{thm:dwl cot}, we first note that $d^{c^k}$ can be recursively computed.
For each $i=0,\dots,k$, we define $V_i:X^{i+1}\times Y^{i+1}\rightarrow\R$ recursively as follows:
\begin{align*}
    &V_{k}(x_0,\dots,x_{k},y_0,\dots,y_{k}):=d_Z(\ell_X(x_k),\ell_Y(y_k))\text{, and}\\
    &V_{i-1}(x_0,\dots,x_{i-1},y_0,\dots,y_{i-1}):=\\
    &\inf_{\nu^i\in\mathcal{C}(\alpha^k_{x_0,\dots,x_{i-1}},\beta^k_{y_0,\dots,y_{i-1}})}\!\!\!\!\!\!\!\!\mathbb{E}_{\nu^i}\,V_i(X_0,\ldots,X_{i-1},Y_0,\ldots,Y_{i-1})
\end{align*}
where $\alpha^k_{x_0,\dots,x_{i-1}}$ (resp. $\beta^k_{y_0,\dots,y_{i-1}}$) is the probability measure on $X$ (resp. $Y$) defined as the conditional probability measure $\alpha^k_{x_0,\dots,x_{i-1}}(x_i):=\alpha^k(x_i\vert x_0,\dots,x_{i-1})$ (resp. $\beta^k_{y_0,\dots,y_{i-1}}(y_i):=\beta^k(y_i\vert y_0,\dots,y_{i-1})$)\footnote{ It might happen that the event $(X_0,\ldots,X_{i-1})=(x_0,\ldots,x_{i-1})$ is null, then the conditional probability is not defined. In this case, we simply let $\alpha^k_{x_0,\dots,x_{i-1}}$ be any probability measure on $X$. We adopt the same convention for $\beta^k$.  An alternative and more rigorous way of dealing with this is to define $\alpha^k_{x_0,\dots,x_{i-1}}$ as the disintegration of $\alpha^k$ w.r.t. $(x_0,\ldots,x_{i-1})$ {(cf. \citep[Section 5.3]{ambrosio2005gradient})}. However, for simplicity of presentation, we avoid such a definition.}. 

Then, as a direct consequence of Proposition 5.2 in \citep{backhoff2017causal},
one has the following equality 
\begin{equation}\label{eq:dcaltdef}
    d^{c^k}(\alpha^k,\beta^k)=\inf_{\gamma^0\in\mathcal{C}(\mu_X,\mu_Y)}\int_{X\times Y}V_0(x_0,y_0)\,\gamma^0(dx_0\times dy_0).
\end{equation}

As we are dealing with Markov chains, it is expected that $V_i(x_0,\ldots,x_{i-1},y_0,\ldots,y_{i-1})$ is independent of the past, i.e., independent of $(x_0,\ldots,x_{i-2},y_0,\ldots,y_{i-2})$. 
In order to show that this is indeed the case, for each $i=0,\dots,k-1$, we define $W_i:X\times Y\rightarrow\R$ in a way similar to how we defined $V_i$:
\begin{align*}
    &W_{k}(x_{k},y_{k}):=d_Z(\ell_X(x_k),\ell_Y(y_k))\text{, and}\\
    &W_{i-1}(x_{i-1},y_{i-1}):=\inf_{\nu^i\in\mathcal{C}(\alpha^k_{x_{i-1}},\beta^k_{y_{i-1}})}\mathbb{E}_{\nu^i}\,W_i(X_i,Y_i)
\end{align*}
for each $i=0,\dots,k-1$ where $\alpha^k_{x_{i-1}}$ (resp. $\beta^k_{y_{i-1}}$) is the probability measure on $X$ (resp. $Y$) defined as the conditional probability measure $\alpha^k_{x_{i-1}}(x_i):=\alpha^k(x_i\vert x_{i-1})$ (resp. $\beta^k_{y_{i-1}}(y_i):=\beta^k(y_i\vert y_{i-1})$).

\begin{lemma}\label{VWsame}
For all $i\in\{0,\dots,k-1\}$ and $(x_0,\dots,x_i,y_0,\dots,y_i)\in X^{i+1}\times Y^{i+1}$, we have the following equality:
$$V_i(x_0,\dots,x_i,y_0,\dots,y_i)=W_i(x_i,y_i).$$
\end{lemma}

Finally, as a direct consequence of \Cref{thm:dwlstochastic}, \Cref{lemma:MCtoBC}, and \Cref{VWsame}, one can prove \Cref{thm:dwl cot}. See \Cref{sec:proof of cot} for a complete proof.

\begin{remark}[Markov chains in the label space]
In fact, one can also transform a $Z$-LMMC $(\mX,\ell_X)$ into a Markov chain in the label space $Z$, if we require that $\ell_X:X\rightarrow Z$ is injective. 
In this way, the WL distance can be also interpreted as solving a COT problem in the label space $Z$. See \Cref{sec: wl in Z} for more details.
\end{remark}

\section{Implications to Message Passing GNNs}\label{sec: GNN}

In this section, we will provide some results on the use of the WL distance for studying the universality and the Lipschitz property of message passing GNNs. In \citep{chen2022weisfeilerlehman}, the authors introduced a neural network framework, named Markov chain neural network (MCNN), on the collection of all LMMCs and established that this framework is universal w.r.t. the WL distance on this collection. It is briefly touched upon in \citep{chen2022weisfeilerlehman} that a MCNN will reduce to a standard Message Passing Graph Neural Network (MP-GNN) when the input LMMCs are induced from labeled graphs. 
In this section, we will restrict ourselves to graph induced LMMCs and study properties of MP-GNNs via the use of the WL distances. In particular, we show that a special yet common type of MP-GNNs
(1) has the same power as the WL distance in distinguishing labeled graphs,
(2) is Lipschitz w.r.t. change of input labeled graphs through the lens of the WL distance and (3) satisfies the universal approximation property, i.e., any continuous function defined on a compact space of labeled graphs can be approximated by such MP-GNNs.

\subsection{Graph induced LMMCs}\label{sec: graph induced LMMC}
We first introduce some terminology related to graphs. In this paper, we consider finite edge weighted graph $G=(V,E, w:E\rightarrow (0,\infty))$ where $V$ denotes the vertex set, $E$ denotes the edge set and $w$ denotes the edge weight function. For each $v\in V$, we define its degree as $\deg(v):=\sum_{v'\in V}w_{vv'}$.
In order to distinguish between multiple graphs, we sometimes include the graph symbol $G$ in subscripts when referring to these notions, such as in the case of $V_G$ or $w_G$. This helps to clarify which graph we are referring to in cases where multiple graphs are involved.

Now, given a finite edge weighted graph $G$ endowed with a label function $\ell_G:V\rightarrow Z$, one can generate a LMMC as follows.
Given $q\in(0,1)$, we associate to the vertex set $V$ a Markov kernel $m_\bullet^{G,q}$ as follows: for any $v\in V$,
\[m_v^{G,q}:=\begin{cases}
    q\delta_v+\frac{1-q}{\deg(v)}\sum_{v'\in N_G(v)}w_{vv'}\delta_{v'}, & \deg(v)>0;\\
    \delta_v, & \deg(v)=0.
\end{cases}\]
For each vertex $v\in V$, we introduce the following modified notion of degree
\begin{equation}\label{eq: modified degree}
    \overline{\deg}(v):=\begin{cases}
    \deg(v)&\deg(v)>0;\\
    1&\deg(v)=0.
\end{cases}
\end{equation}
Based on $\overline{\deg}$, we introduce the following probability measure on $V$: $\mu_G:=\sum_{v\in V}\frac{\overline{\deg}(v)}{\sum_{v'\in V}\overline{\deg}(v')}\delta_v$. 
Then for any $q\in (0,1)$, $\mu_G$ is a stationary distribution w.r.t. the Markov kernel $m_\bullet^{G,q}$. Then, we let $\mX_q(G):=(V,m_\bullet^{G,q},\mu_G)$ and we say that the LMMC $(\mX_q(G),\ell_G)$ is induced by the labeled graph $(G,\ell_G)$.

Finally, we let $\mathcal{G}(Z)$ denote the collection of all $Z$-labeled graphs and given $q>0$, let $I_q:\mathcal{G}(Z)\rightarrow \mathcal{M}^L(Z)$ denote the map which sends a $Z$-labeled graph into a $Z$-LMMC via the method described above.
Let $\mathcal{G}_q(Z):=I_q(\mathcal{G}(Z))\subseteq\mathcal{M}^L(Z)$. Then, for any $k\geq 0$, $\dWLk$ restricted on $\mathcal{G}_q(Z)$ induces a pseudo-distance, which we denote by $d_{\mathcal{G},q}^{\scriptscriptstyle{(k)}}$, on $\mathcal{G}(Z)$. We then call $d_{\mathcal{G},q}^{\scriptscriptstyle{(k)}}$ the ($q$-damped) WL distance of depth $k$ between labeled graphs.

\subsection{Message Passing Graph Neural Networks}\label{sec:MPNN}
Given $q>0$ we consider the following special type of $k$-layer MP-GNNs.
\begin{align*}
 &\text{Message Passing:}\quad\ell_G^{i+1}(v):=\\
 &\begin{cases}q\,\varphi_{i+1}(\ell_G^i(v))+\frac{1-q}{\deg(v)}\sum_{v'\in N_G(v)}{w_{vv'}}\varphi_{i+1}(\ell^i_G(v'))\\
 \qquad\text{if }\deg(v)>0,\\
 \varphi_{i+1}(\ell_G^i(v))\text{ if }\deg(v)=0\end{cases}
\end{align*}
\begin{align*}
 &\text{Readout:}\\
 &h((G,\ell_G)):=\psi\lc \sum_{v\in V}\frac{\overline{\deg}(v)}{\sum_{v'\in V} \overline{\deg}(v')}\varphi_{k+1}(\ell^k_G(v))\rc
\end{align*}
where $\varphi_i:\R^{d_{i-1}}\rightarrow\R^{d_i}$ and $\psi:\R^{d_{k+1}}\rightarrow \R$ are MLPs.

\begin{remark}\label{rmk:restriction}
Any such MP-GNN $h:\mathcal{G}(Z)\rightarrow \R$ arises as the restriction of a Markov Chain Neural Network (MCNN) $H:\mathcal{M}^L(Z)\rightarrow \R$ (see \citep[Section 4]{chen2022weisfeilerlehman}) to the collection $\mathcal{G}_q(Z)$. For more details, see \Cref{sec:MCNN}.
\end{remark}

We use ${\mathcal{N\!N}_k^q(\R^d)}$ to denote the collection of all such MP-GNNs with $k$ layers.

\paragraph{Discriminative Power of MP-GNNs} 
In addition to the Lipschitz property, we also establish that $\mathcal{N\!N}_k^q(\R^d)$ has the same discriminative power as the WL distance. 

\begin{proposition}\label{prop:zero set of NN}
Given any $(G_1,\ell_{G_1}),(G_2,\ell_{G_2})\in \mathcal{G}(\R^d)$, 
\begin{enumerate}
    \item if $d_{\mathcal{G},q}^{\scriptscriptstyle{(k)}}\!\lc(G_1,\ell_{G_1}),(G_2,\ell_{G_2})\rc =0$, then for every $h\in \mathcal{N\!N}^q_k(\R^d)$ one has that $h((G_1,\ell_{G_1}))=h((G_2,\ell_{G_2}))$;
    \item if $d_{\mathcal{G},q}^{\scriptscriptstyle{(k)}}\!\lc(G_1,\ell_{G_1}),(G_2,\ell_{G_2})\rc >0$, then there exists $h\in \mathcal{N\!N}_k(\R^d)$ such that $h((G_1,\ell_{G_1}))\neq h((G_2,\ell_{G_2}))$.
\end{enumerate}
\end{proposition}
\begin{proof}
    This follows directly from \Cref{rmk:restriction} and  \citep[Proposition 4.1]{chen2022weisfeilerlehman}.
\end{proof}

In fact, we establish a fact stronger than \Cref{prop:zero set of NN} under a setting similar to (but more flexible than) the one used in \citep{xu2018powerful}. 
Choose any \emph{countable} subset $Z\subseteq\R^d$ and any \emph{countable} subset $P\subseteq \R$. Let $\mathcal{G}_P(Z)$ denote the collection of all {$Z$-}labeled weighted graphs so that their edge weights are contained in $P$.

Now, we will establish that a very restricted set of MP-GNNs is sufficient to have the same discriminative power as the WL distance and, in this way, we establish a much stronger result than \Cref{prop:zero set of NN}. We let $\mathcal{N\!N}_k^{q,1}(\R^d)$ denote the collection of maps $h\in \mathcal{N\!N}_k^{q}(\R^d)$
where $\varphi_i:\R^{d_{i-1}}\rightarrow\R^{d_{i}}$ satisfies that $d_i=1$ for each $i=1,\ldots,k+1$ (note $d_0=d$).
Then, we establish the following main result:
\begin{theorem}\label{thm: better than gin}
For each $k\geq 0$, there exists $h\in\mathcal{N\!N}_k^{q,1}(\R^d)$ such that for any $(G_1,\ell_{G_1}),(G_2,\ell_{G_2})\in \mathcal{G}_P(Z)$, $d_{\mathcal{G},q}^{\scriptscriptstyle{(k)}}\!\lc(G_1,\ell_{G_1}),(G_2,\ell_{G_2})\rc >0$ iff $h((G_1,\ell_{G_1}))\neq h((G_2,\ell_{G_2}))$.
\end{theorem}
Notice that whereas in \Cref{prop:zero set of NN} the existence of $h$ may depend on the choice of labeled graphs, $h$ in \Cref{thm: better than gin} is universal for all pairs of labeled graphs. Furthermore, while maps $\varphi_i$ involved in intermediate layers of $h$ in \Cref{prop:zero set of NN} could potentially have large dimensions ($d_i$ could be very large), $\varphi_i$ can be chosen to have very low dimension ($d_i=1$) in \Cref{thm: better than gin}. In this way, the latter result is stronger than the previous one.
\begin{remark}\label{rmk:GIN countable}
\cite{xu2018powerful} established a result similar to \Cref{thm: better than gin}: for any $k\geq 0$, there exists a MP-GNN $h$ which has the same discriminative power as the $k$-step WL test when restricted to the set of graphs whose labels are from a common \textit{countable} set. We remark that in order to show the existence of such $h$, however, the MP-GNN $h$ they constructed utilizes aggregation functions defined on a \textit{countable} set which may not be able to be continuously extended to a ``continuous'' domain such as $\R^d$. In contrast, each function involved in constructing $h\in\mathcal{N\!N}_k^{q,1}(\R^d)$ is a continuous map between Euclidean spaces. 
\end{remark}

\paragraph{Lipschitz property of MP-GNNs}
In addition to the study of zero sets, one can generalize the results above in a quantitative manner. More specifically, we establish that the specific MP-GNNs defined in this section 
are Lipschitz w.r.t. the WL distance. This in particular indicates that MP-GNNs are stable w.r.t. small perturbations of graphs in the sense of the WL distance.
\begin{theorem}\label{thm: lip}
Given a $k$-layer MP-GNN $h:\mathcal{G}(\R^d)\rightarrow\R$ as described above, assume that $\varphi_i$ is $C_i$-Lipschitz for $i=1,\ldots,k+1$ and that $\psi$ is $C$-Lipschitz.  
Then, for any two labeled graphs $(G_1,\ell_{G_1})$ and $(G_2,\ell_{G_2})$, one has that
\begin{align*}
    &|h((G_1,\ell_{G_1}))-h((G_2,\ell_{G_2}))|\\
    &\leq C\cdot\Pi_{i=1}^{k+1}C_i\cdot d_{\mathcal{G},q}^{\scriptscriptstyle{(k)}}\!\lc(G_1,\ell_{G_1}),(G_2,\ell_{G_2})\rc.
\end{align*}
\end{theorem}

\begin{remark}
    We note that the WL distance could potentially be used to study the Lipschitz property of other types of MP-GNNs. See \Cref{sec: lip} for such a study of a normalized version of Graph Isomorphism Network.
\end{remark}

\paragraph{Universal approximation} 

Based on the Lipschitz property, we finally establish the universal approximation property of MP-GNNs. 
For this purpose, we introduce some notation.
Given any $k\in\N$ and any subset $\mathcal{K}\subseteq \lc\mathcal{G}(\R^d),d_{\mathcal{G},q}^{\scriptscriptstyle{(k)}}\rc$, we let $C(\mathcal{K},\R)$ denote the set of all continuous functions $f:\mathcal{K}\rightarrow \R$.
We further let $\mathcal{N\!N}_k^q(\R^d)|_\mathcal{K}$ denote the collection of all  functions $h|_\mathcal{K}$ where $h\in \mathcal{N\!N}_k^q(\R^d)$, i.e.,
\[\mathcal{N\!N}_k^q(\R^d)|_\mathcal{K}:=\left\{h|_\mathcal{K}:\,h\in \mathcal{N\!N}_k^q(\R^d)\right\}.\]
Then, we state our main result as follows.

\begin{theorem}[Universal approximation of $k$-layer MP-GNNs]\label{thm: universal}
Given $q>0$ and any $k\in\mathbb{N}$, let $\mathcal{K}\subseteq \lc \mathcal{G}(\R^d),d_{\mathcal{G},q}^{\scriptscriptstyle{(k)}}\rc$ be any compact subspace. Then, $\overline{\mathcal{N\!N}_k^q(\R^d)|_\mathcal{K}}=C(\mathcal{K},\mathbb{R})$.
\end{theorem}
The theorem above follows directly from the universal approximation result for MCNNs \citep[Theorem 4.3]{chen2022weisfeilerlehman}; see \Cref{app:proof universal} for the proof.


\section{Conclusion and Future Directions}

In this paper, we further investigate the WL distance, proposed in \citep{chen2022weisfeilerlehman}, for comparing LMMCs and establish that the WL distance of depth $k$ can be seen as a variant of the $\ell_1$-Wasserstein distance comparing distributions of trajectories from random walks on the label space.
We further identify connections between the WL distance and causal optimal transport (COT), suggesting potential applications in computing/approximating COT distances and extending the WL distance to handle different label spaces. These avenues offer interesting future research directions.

Given that the WL distance is compatible with the WL-graph isomorphism test, which has been connected to the expressiveness of message-passing graph neural networks (MP-GNNs) \citep{xu2018powerful,azizian2020expressive}, it is natural to equip the space of graphs with a (pseudo-)metric structure induced by the WL distance. As already observed in \citep{chen2022weisfeilerlehman} and further detailed in this paper, MP-GNN can universally approximate continuous functions defined on the space of graphs w.r.t. the WL distance. In fact, a more refined result is obtained (Theorem \ref{thm: universal}) which shows universal approximation result for the family of $k$-layer MP-GNNs. 
Furthermore, similar to the work of \citep{chuang2022tree}, we show that the WL distance can also be used to study stability and Lipschitz property of sub-families of MP-GNNs. We also remark that the WL distance can serve as a suitable choice of metric for the space of graphs to study questions such as the generalization power of message passing GNNs as done in \citep{chuang2022tree}.  

\paragraph*{Acknowledgement.} This work is partially supported by the NSF through grants 
IIS-1901360, IIS-2050360, CCF-1740761, CCF-2112665, CCF-2310412, DMS-1547357,   and by the BSF under grant 2020124.


\bibliography{biblio-dWL-workshop}

\begin{thebibliography}{24}
\providecommand{\natexlab}[1]{#1}
\providecommand{\url}[1]{\texttt{#1}}
\expandafter\ifx\csname urlstyle\endcsname\relax
  \providecommand{\doi}[1]{doi: #1}\else
  \providecommand{\doi}{doi: \begingroup \urlstyle{rm}\Url}\fi

\bibitem[Ambrosio et~al.(2005)Ambrosio, Gigli, and
  Savar{\'e}]{ambrosio2005gradient}
Ambrosio, L., Gigli, N., and Savar{\'e}, G.
\newblock \emph{Gradient flows: in metric spaces and in the space of
  probability measures}.
\newblock Springer Science \& Business Media, 2005.

\bibitem[Azizian et~al.(2020)]{azizian2020expressive}
Azizian, W. et~al.
\newblock Expressive power of invariant and equivariant graph neural networks.
\newblock In \emph{International Conference on Learning Representations}, 2020.

\bibitem[Backhoff et~al.(2017)Backhoff, Beiglbock, Lin, and
  Zalashko]{backhoff2017causal}
Backhoff, J., Beiglbock, M., Lin, Y., and Zalashko, A.
\newblock Causal transport in discrete time and applications.
\newblock \emph{SIAM Journal on Optimization}, 27\penalty0 (4):\penalty0
  2528--2562, 2017.

\bibitem[Backhoff-Veraguas et~al.(2020)Backhoff-Veraguas, Bartl, Beiglb{\"o}ck,
  and Eder]{backhoff2020adapted}
Backhoff-Veraguas, J., Bartl, D., Beiglb{\"o}ck, M., and Eder, M.
\newblock Adapted {W}asserstein distances and stability in mathematical
  finance.
\newblock \emph{Finance and Stochastics}, 24\penalty0 (3):\penalty0 601--632,
  2020.

\bibitem[Brug{\`e}re et~al.(2023)Brug{\`e}re, Wan, and
  Wang]{brugere2023distances}
Brug{\`e}re, T., Wan, Z., and Wang, Y.
\newblock Distances for {M}arkov chains, and their differentiation.
\newblock \emph{arXiv preprint arXiv:2302.08621}, 2023.

\bibitem[Burago et~al.(2001)Burago, Burago, and Ivanov]{burago2001course}
Burago, D., Burago, Y., and Ivanov, S.
\newblock \emph{A course in metric geometry}, volume~33.
\newblock American Mathematical Soc., 2001.

\bibitem[Chen et~al.(2022)Chen, Lim, Mémoli, Wan, and
  Wang]{chen2022weisfeilerlehman}
Chen, S., Lim, S., Mémoli, F., Wan, Z., and Wang, Y.
\newblock Weisfeiler-{L}ehman meets {G}romov-{W}asserstein.
\newblock In \emph{International Conference on Machine Learning (ICML)}, pp.\
  3371--3416. PMLR, 2022.

\bibitem[Chuang \& Jegelka(2022)Chuang and Jegelka]{chuang2022tree}
Chuang, C.-Y. and Jegelka, S.
\newblock Tree mover's distance: Bridging graph metrics and stability of graph
  neural networks.
\newblock \emph{arXiv preprint arXiv:2210.01906}, 2022.

\bibitem[Durrett(2019)]{durrett2019probability}
Durrett, R.
\newblock \emph{Probability: theory and examples}, volume~49.
\newblock Cambridge university press, 2019.

\bibitem[Eckstein \& Pammer(2022)Eckstein and
  Pammer]{eckstein2022computational}
Eckstein, S. and Pammer, G.
\newblock Computational methods for adapted optimal transport.
\newblock \emph{arXiv preprint arXiv:2203.05005}, 2022.

\bibitem[Glanzer et~al.(2019)Glanzer, Pflug, and
  Pichler]{glanzer2019incorporating}
Glanzer, M., Pflug, G.~C., and Pichler, A.
\newblock Incorporating statistical model error into the calculation of
  acceptability prices of contingent claims.
\newblock \emph{Mathematical Programming}, 174\penalty0 (1):\penalty0 499--524,
  2019.

\bibitem[Klemmer et~al.(2022)Klemmer, Xu, Acciaio, and Neill]{klemmer2022spate}
Klemmer, K., Xu, T., Acciaio, B., and Neill, D.~B.
\newblock {SPATE-GAN}: Improved generative modeling of dynamic spatio-temporal
  patterns with an autoregressive embedding loss.
\newblock In \emph{Proceedings of the AAAI Conference on Artificial
  Intelligence}, volume~36, pp.\  4523--4531, 2022.

\bibitem[Kolmogorov \& Bharucha-Reid(2018)Kolmogorov and
  Bharucha-Reid]{kolmogorov2018foundations}
Kolmogorov, A.~N. and Bharucha-Reid, A.~T.
\newblock \emph{Foundations of the theory of probability: Second English
  Edition}.
\newblock Courier Dover Publications, 2018.

\bibitem[Lassalle(2018)]{lassalle2018causal}
Lassalle, R.
\newblock Causal transport plans and their {M}onge--{K}antorovich problems.
\newblock \emph{Stochastic Analysis and Applications}, 36\penalty0
  (3):\penalty0 452--484, 2018.

\bibitem[Lehman \& Weisfeiler(1968)Lehman and Weisfeiler]{leman1968reduction}
Lehman, A. and Weisfeiler, B.
\newblock A reduction of a graph to a canonical form and an algebra arising
  during this reduction.
\newblock \emph{Nauchno-Technicheskaya Informatsiya}, 2\penalty0 (9):\penalty0
  12--16, 1968.

\bibitem[Levin \& Peres(2017)Levin and Peres]{levin2017markov}
Levin, D.~A. and Peres, Y.
\newblock \emph{Markov chains and mixing times}, volume 107.
\newblock American Mathematical Soc., 2017.

\bibitem[M{\'e}moli(2011)]{memoli2011gromov}
M{\'e}moli, F.
\newblock Gromov-{W}asserstein distances and the metric approach to object
  matching.
\newblock \emph{Foundations of computational mathematics}, 11\penalty0
  (4):\penalty0 417--487, 2011.

\bibitem[Pichler \& Weinhardt(2022)Pichler and Weinhardt]{pichler2022nested}
Pichler, A. and Weinhardt, M.
\newblock The nested sinkhorn divergence to learn the nested distance.
\newblock \emph{Computational Management Science}, 19\penalty0 (2):\penalty0
  269--293, 2022.

\bibitem[Togninalli et~al.(2019)Togninalli, Ghisu, Llinares-L{\'o}pez, Rieck,
  and Borgwardt]{togninalli2019wasserstein}
Togninalli, M., Ghisu, E., Llinares-L{\'o}pez, F., Rieck, B., and Borgwardt, K.
\newblock Wasserstein {W}eisfeiler-{L}ehman graph kernels.
\newblock \emph{Advances in Neural Information Processing Systems},
  32:\penalty0 6439--6449, 2019.

\bibitem[Toth et~al.(2022)Toth, Lee, Hacker, and Oberhauser]{toth2022capturing}
Toth, C., Lee, D., Hacker, C., and Oberhauser, H.
\newblock Capturing graphs with hypo-elliptic diffusions.
\newblock \emph{36th Conference on Neural Information Processing Systems
  (NeurIPS 2022)}, 2022.

\bibitem[Weaver(1995)]{weaver1995order}
Weaver, N.
\newblock Order completeness in {L}ipschitz algebras.
\newblock \emph{Journal of Functional Analysis}, 130\penalty0 (1):\penalty0
  118--130, 1995.

\bibitem[Xu et~al.(2018)Xu, Hu, Leskovec, and Jegelka]{xu2018powerful}
Xu, K., Hu, W., Leskovec, J., and Jegelka, S.
\newblock How powerful are graph neural networks?
\newblock In \emph{International Conference on Learning Representations}, 2018.

\bibitem[Xu \& Acciaio(2022)Xu and Acciaio]{xu2022CCOT}
Xu, T. and Acciaio, B.
\newblock Conditional {COT}-{GAN} for video prediction with kernel smoothing.
\newblock \emph{Workshop on Robustness in Sequence Modeling, 36th Conference on
  Neural Information Processing Systems}, 2022.

\bibitem[Xu et~al.(2020)Xu, Wenliang, Munn, and Acciaio]{xu2020cot}
Xu, T., Wenliang, L.~K., Munn, M., and Acciaio, B.
\newblock {COT}-{GAN}: Generating sequential data via causal optimal transport.
\newblock \emph{Advances in Neural Information Processing Systems},
  33:\penalty0 8798--8809, 2020.

\end{thebibliography}
\bibliographystyle{icml2023}


\newpage
\onecolumn
\appendix

\section{Details from \Cref{sec:sec4}}
\subsection{Proof of \Cref{prop:mcoupling=kcoupling}}\label{sec:proof of prop:mcoupling=kcoupling}
To facilitate the proof below, we consider the canonical identification $\iota:(X\times Y)^\N\rightarrow X^\N\times Y^\N$ sending $(x_j,y_j)_{j\in\N}$ to $((x_j)_{j\in\N},(y_j)_{j\in\N})$. 
This is obviously a homeomorphism given the product topology. 
By slight abuse of notation, we let $X_i$ denote both the projection map $X_i:X^\N\times Y^\N\rightarrow X$ to the $i$-th component for $i\in\N$ and also the projection map $X_i\circ\iota:(X\times Y)^\N\rightarrow X$. We adopt a similar convention for $Y_i$.

For any $k$-step coupling $\gamma^{\scriptscriptstyle{(k)}}\in\cpl^{\scriptscriptstyle{(k)}}(\mu_X,\mu_Y)$,  there exist $\gamma\in\mathcal{C}(\mu_X,\mu_Y)$ and $(\nu_i)_{\bullet,\bullet}\in\cpl^{\scriptscriptstyle{(1)}}\!\lc m_\bullet^X,m_\bullet^Y\rc$ for $i=1,\ldots,k$ such that
\[\gamma^{\scriptscriptstyle{(k)}}=\int\limits_{X\times Y}\cdots\int\limits_{X\times Y}(\nu_k)_{x_{k-1},y_{k-1}}\,(\nu_{k-1})_{x_{k-2},y_{k-2}}(dx_{k-1}\times dy_{k-1})\cdots(\nu_1)_{x_0,y_0}(dx_1\times dy_1)\gamma(dx_0\times dy_0).\]
We further let $(\nu_i)_{\bullet,\bullet}:=(\nu_k)_{\bullet,\bullet}$ for all $i>k$.
Now, by the Kolmogorov extension theorem, there exists $\mathbb{Q}\in\mathcal{P}((X\times Y)^\N)$ such that 
\begin{enumerate}
    \item $(\nu_{i+1})_{x,y}(x',y')=\mathbb{Q}\big((X_{i+1},Y_{i+1})=(x',y')|(X_{i},Y_{i})=(x,y)\big),$ 
for any $x,x'\in X$, $y,y'\in Y$ and any $i=0,\ldots$. 
\item $\mathbb{Q}((X_i,Y_i)\in B_i,0\leq i\leq n)=\int\limits_{B_0}\gamma(dx_0\times dy_0)\cdots\int\limits_{B_n}(\nu_n)_{x_{n-1},y_{n-1}}(dx_{n}\times dy_{n})$ for any $n\in\N$ and any measurable $B_i\subseteq X\times Y$ for $i=0,\ldots,n$.
\end{enumerate}

Now, we let $\PP:=\iota_\#\mathbb{Q}\in\mathcal{P}(X^\N\times Y^\N)$. Then, it is straightforward to check that 
\begin{enumerate}
    \item $(\nu_{i+1})_{x,y}(x',y')=\mathbb{P}\big((X_{i+1},Y_{i+1})=(x',y')|(X_{i},Y_{i})=(x,y)\big),$ 
for any $x,x'\in X$, $y,y'\in Y$ and any $i=0,\ldots$. 
\item $\mathbb{P}((X_i,Y_i)\in B_i,0\leq i\leq n)=\int\limits_{B_0}\gamma(dx_0\times dy_0)\cdots\int\limits_{B_n}(\nu_n)_{x_{n-1},y_{n-1}}(dx_{n}\times dy_{n})$ for any $n\in\N$ and any measurable $B_i\subseteq X\times Y$ for $i=0,\ldots,n$.
\end{enumerate}
In this way, it is easy to check that $\PP\in\mathcal{C}_\mathcal{M}(\PP_X,\PP_Y)$.
Finally, we have that 
\begin{align*}
    &\mathrm{law}((X_k,Y_k))=(X_k,Y_k)_\#\mathbb{P}\\
 =&\int\limits_{X\times Y}\cdots\int\limits_{X\times Y}(\nu_k)_{x_{k-1},y_{k-1}}\,(\nu_{k-1})_{x_{k-2},y_{k-2}}(dx_{k-1}\times dy_{k-1})\cdots(\nu_1)_{x_0,y_0}(dx_1\times dy_1)\gamma(dx_0\times dy_0)\\
 =&\gamma^{\scriptscriptstyle{(k)}}.
\end{align*}

For the other direction, given a Markovian coupling $\PP\in\mathcal{C}_\mathcal{M}(\PP_X,\PP_Y)$, we let
$$(\nu_{i+1})_{x,y}(x',y'):=\mathbb{P}\big((X_{i+1},Y_{i+1})=(x',y')|(X_{i},Y_{i})=(x,y)\big)$$
for any $i\in\N$, $x,x'\in X$ and $y,y'\in Y$.\footnote{Recall from \Cref{def: markov coupling} that when the event $\{(X_i,Y_i)=(x,y)\}$ is null, we let $(\nu_{i+1})_{x,y}:=m_x^X\otimes m_y^Y$.}
Then, by definition of the Markovian coupling, $(\nu_{i+1})_{\bullet,\bullet}$ is a 1-step coupling between $m_\bullet^X$ and $m_\bullet^Y$ for each $i=0,1,\ldots$. Let $\gamma:=\mathrm{law}((X_0,Y_0))\in\mathcal{C}(\mu_X,\mu_Y)$. 
Then, if we let
\[\nu_{x_{0},y_{0}}^{\scriptscriptstyle{(k)}}\coloneqq\int\limits_{X\times Y}\cdots\int\limits_{X\times Y}(\nu_k)_{x_{k-1},y_{k-1}}\,(\nu_{k-1})_{x_{k-2},y_{k-2}}(dx_{k-1}\times dy_{k-1})\cdots(\nu_1)_{x_0,y_0}(dx_1\times dy_1)\]
for any $x_{0}\in X$ and $y_{0} \in Y$, then $\nu_{\bullet,\bullet}^{\scriptscriptstyle{(k)}}\in\cpl^{\scriptscriptstyle{(k)}}\!\lc m_\bullet^{X},m_\bullet^Y\rc$ and thus
\begin{align*}
    \mathrm{law}((X_k,Y_k))=(X_k,Y_k)_\#\mathbb{P}=\nu_{\bullet,\bullet}^{\scriptscriptstyle{(k)}}\odot \gamma.
\end{align*}

\subsection{Proof of \Cref{thm:dwlstochastic}}
Fix an arbitrary $k$-step coupling $\gamma^{\scriptscriptstyle{(k)}}=\nu_{\bullet,\bullet}^{\scriptscriptstyle{(k)}}\odot \gamma\in \cpl^{\scriptscriptstyle{(k)}}(\mu_X,\mu_Y)$. Then, by \Cref{prop:mcoupling=kcoupling}, there exists a Markovian coupling $\PP\in\mathcal{C}_\mathcal{M}(\PP_X,\PP_Y)$ such that 
$$\mathrm{law}((X_k,Y_k))=\gamma^{\scriptscriptstyle{(k)}}\,\,\text{and}\,\,\mathrm{law}((X_0,Y_0))=\gamma.$$
This implies that
\begin{align*}
    &\int_{X\times Y}d_Z(\ell_X(x),\ell_Y(y))\gamma^{\scriptscriptstyle{(k)}}(dx\times dy)=\int_{X\times Y}d_Z(\ell_X(x),\ell_Y(y))(X_k,Y_k)_\#\PP(dx\times dy)
    \\&=\int_{X^\N\times Y^\N}\, d_Z(\ell_X({X}_k(w_X)),\ell_Y({Y}_k(w_Y)))\PP(dw_X\times dw_Y)\\
    &\geq\inf_{\PP\in\mathcal{C}_\mathcal{M}(\PP_X,\PP_Y)}\int_{X^\N\times Y^\N}\, d_Z(\ell_X({X}_k(w_X)),\ell_Y({Y}_k(w_Y)))\PP(dw_X\times dw_Y).
\end{align*}
Since $\gamma^{\scriptscriptstyle{(k)}}$ is arbitrary, we have that
\begin{align*}
    &\inf_{\gamma^{\scriptscriptstyle{(k)}}\in \cpl^{\scriptscriptstyle{(k)}}(\mu_X,\mu_Y)} \int\limits_{X\times Y}d_Z(\ell_X(x),\ell_Y(y))\gamma^{\scriptscriptstyle{(k)}}(dx\times dy)\\
    &\geq\inf_{\PP\in\mathcal{C}_\mathcal{M}(\PP_X,\PP_Y)}\int_{X^\N\times Y^\N}\, d_Z(\ell_X({X}_k(w_X)),\ell_Y({Y}_k(w_Y)))\PP(dw_X\times dw_Y).
\end{align*}

Conversely, fix an arbitrary Markovian coupling $\PP\in\mathcal{C}_\mathcal{M}(\PP_X,\PP_Y)$. Then, by \Cref{prop:mcoupling=kcoupling}, $\gamma^{\scriptscriptstyle{(k)}}:=\mathrm{law}((X_k,Y_k))$ is a $k$-step coupling between $\mu_X$ and $\mu_Y$. Therefore,
\begin{align*}
    &\int_{X^\N\times Y^\N}\, d_Z(\ell_X({X}_k(w_X)),\ell_Y({Y}_k(w_Y)))\PP(dw_X\times dw_Y)=\int\limits_{X\times Y}d_Z(\ell_X(x),\ell_Y(y))\gamma^{\scriptscriptstyle{(k)}}(dx\times dy)\\   &\geq\inf_{\gamma^{\scriptscriptstyle{(k)}}\in \cpl^{\scriptscriptstyle{(k)}}(\mu_X,\mu_Y)} \int\limits_{X\times Y}d_Z(\ell_X(x),\ell_Y(y))\gamma^{\scriptscriptstyle{(k)}}(dx\times dy).
\end{align*}

Since $\PP$ is arbitrary,
\begin{align*}
    &\inf_{\gamma^{\scriptscriptstyle{(k)}}\in \cpl^{\scriptscriptstyle{(k)}}(\mu_X,\mu_Y)} \int\limits_{X\times Y}d_Z(\ell_X(x),\ell_Y(y))\gamma^{\scriptscriptstyle{(k)}}(dx\times dy)\\
    &\leq\inf_{\PP\in\mathcal{C}_\mathcal{M}(\PP_X,\PP_Y)}\int_{X^\N\times Y^\N}\, d_Z(\ell_X({X}_k(w_X)),\ell_Y({Y}_k(w_Y)))\PP(dw_X\times dw_Y).
\end{align*}

Finally, by \Cref{thm:dwl= dwk}, one can conclude that
\begin{align*}
    \dWLk\!\lc(\mX,\ell_X),(\mY,\ell_Y)\rc &=\inf_{\gamma^{\scriptscriptstyle{(k)}}\in \cpl^{\scriptscriptstyle{(k)}}(\mu_X,\mu_Y)} \int\limits_{X\times Y}d_Z(\ell_X(x),\ell_Y(y))\gamma^{\scriptscriptstyle{(k)}}(dx\times dy)\\
    &=\inf_{\PP\in\mathcal{C}_\mathcal{M}(\PP_X,\PP_Y)}\int_{X^\N\times Y^\N}\, d_Z(\ell_X({X}_k(w_X)),\ell_Y({Y}_k(w_Y)))\PP(dw_X\times dw_Y).
\end{align*}

\subsection{Proof of \Cref{lemma:MCtoBC}}
{Note that it is easy to check that $\pi^k$ is indeed a coupling measure between $\alpha^k$ and $\beta^k$ from the definition. Hence, we only need to verify the bicausality of $\pi^k$.} Observe that
\begin{align*}
    &\pi^k((y_0,\dots,y_l)\vert (x_0,\dots,x_k))\\
    &=\mathbb{P}\big(Y_0=y_0,\dots,Y_l=y_l|X_0=x_0,\dots,X_k=x_k\big)\\
    &=\frac{\mathbb{P}\big(X_0=x_0,\dots,X_k=x_k,Y_0=y_0,\dots,Y_l=y_l\big)}{\mathbb{P}\big(X_0=x_0,\dots,X_k=x_k\big)}\\
    &=\frac{\mathbb{P}\big(X_{l+1}=x_{l+1},\dots,X_k=x_k|X_0=x_0,\dots,X_l=x_l,Y_0=y_0,\dots,Y_l=y_l\big)}{\mathbb{P}\big(X_{l+1}=x_{l+1},\dots,X_k=x_k|X_0=x_0,\dots,X_l=x_l\big)}\\
    &\quad\times\frac{\mathbb{P}\big(X_0=x_0,\dots,X_l=x_l,Y_0=y_0,\dots,Y_l=y_l\big)}{\mathbb{P}\big(X_0=x_0,\dots,X_l=x_l\big)}.
\end{align*}
Moreover, since {$\mathbb{P}$} is a Markovian coupling, it is easy to check that
\begin{align*}
    &\mathbb{P}\big(X_{l+1}=x_{l+1},\dots,X_k=x_k|X_0=x_0,\dots,X_l=x_l,Y_0=y_0,\dots,Y_l=y_l\big)\\
    &=\mathbb{P}\big(X_k=x_k|X_0=x_0,\dots,X_{k-1}=x_{k-1},Y_0=y_0,\dots,Y_l=y_l\big)\\
    &\quad\times\mathbb{P}\big(X_{k-1}=x_{k-1}|X_0=x_0,\dots,X_{k-2}=x_{k-2},Y_0=y_0,\dots,Y_l=y_l\big)\\
    &\quad\cdots\times\mathbb{P}\big(X_{l+1}=x_{l+1}|X_0=x_0,\dots,X_l=x_l,Y_0=y_0,\dots,Y_l=y_l\big)\\
    &=m_{x_l}^X(x_{l+1})\cdots m_{x_{k-1}}^X(x_k)\\
    &=\mathbb{P}\big(X_{l+1}=x_{l+1},\dots,X_k=x_k|X_0=x_0,\dots,X_l=x_l\big).
\end{align*}
Therefore, one can conclude that
\begin{align*}
    \pi^k((y_0,\dots,y_l)\vert (x_0,\dots,x_k))&=\frac{\mathbb{P}\big(X_0=x_0,\dots,X_l=x_l,Y_0=y_0,\dots,Y_l=y_l\big)}{\mathbb{P}\big(X_0=x_0,\dots,X_l=x_l\big)}\\
    &=\pi^k((y_0,\dots,y_l)\vert (x_0,\dots,x_l)).
\end{align*}
This implies that $\pi^k$ is causal from $\alpha^k$ to $\beta^k$. In a similar way, one can also prove that $\pi^k$ is causal from $\beta^k$ to $\alpha^k$. This completes the proof.

\subsection{Proof of \Cref{VWsame}}
The proof is by (backward) induction. First, observe that $\alpha^k_{x_0,\dots,x_{i-1}}=\alpha^k_{x_{i-1}}=m^X_{x_{i-1}}$ and $\beta^k_{y_0,\dots,y_{i-1}}=\beta^k_{y_{i-1}}=m^Y_{y_{i-1}}$ since both of $\alpha^k$ and $\beta^k$ are Markovian. The remaining steps are straightforward, so we omit it.

\subsection{Proof of \Cref{thm:dwl cot}}\label{sec:proof of cot}
First, fix an arbitrary Markovian coupling $\PP\in\mathcal{C}_\mathcal{M}(\PP_X,\PP_Y)$.
Then, by \Cref{lemma:MCtoBC}, we know that $\pi^k:=\mathrm{law}((X_0,\ldots,X_k,Y_0,\ldots,Y_k))\in\mathcal{P}(X^{k+1}\times Y^{k+1})$ is a bicausal coupling between $\alpha^k$ and $\beta^k$. Therefore,
$$\int_{X^\N\times Y^\N}\, d_Z(\ell_X({X}_k(w_X)),\ell_Y({Y}_k(w_Y)))\,\PP(dw_X\times dw_Y)=\int_{X^{k+1}\times Y^{k+1}}c^k(x,y)\,\pi^k(dx\times dy)\geq d^{c^k}(\alpha^k,\beta^k).$$
Since the Markovian coupling $\PP$ is chosen arbitrarily, by \Cref{thm:dwlstochastic}, we have $\dWLk\lc(\mX,\ell_X),(\mY,\ell_Y)\rc\geq d^{c^k}(\alpha^k,\beta^k)$.

Now, let us prove the other direction. By \Cref{eq:dcaltdef} and \Cref{VWsame}, one can choose an optimal coupling $\gamma^0\in\mathcal{C}(\mu_X,\mu_Y)$ such that 
\[d^{c^k}(\alpha^k,\beta^k)=\int_{X\times Y}W_0(x_0,y_0)\,\gamma^0(dx_0\times dy_0).\]
Next, for arbitrary $(x_0,y_0)\in X\times Y$, one can choose an optimal coupling $\nu_{x_0,y_0}^1\in\mathcal{C}(m^X_{x_0},m^Y_{y_0})$ such that 
\[W_0(x_0,y_0)=\int_{X\times Y}W_1(x_1,y_1)\,\gamma_{x_0,y_0}^1(dx_1\times dy_1).\]
This implies that $\nu_{\bullet,\bullet}^1$ is a 1-step coupling between $m_\bullet^X$ and $m_\bullet^Y$. We repeat this process inductively. Then, we have the following 1-step couplings $\nu_{\bullet,\bullet}^1,\nu_{\bullet,\bullet}^2,\dots,\nu_{\bullet,\bullet}^k$ between $m_\bullet^X$ and $m_\bullet^Y$ such that 
\[W_{i-1}(x_{i-1},y_{i-1})=\int_{X\times Y}W_i(x_i,y_i)\,\nu_{x_{i-1},y_{i-1}}^i(dx_i\times dy_i)\]
for all $(x_{i-1},y_{i-1})\in X\times Y$. Now, the following equality holds:
$$d^{c^k}(\alpha^k,\beta^k)=\int_{X\times Y}\cdots\int_{X\times Y}d_Z(\ell_X(x_k),\ell_Y(y_k))\,\nu_{x_{k-1},y_{k-1}}^k(dx_k\times dy_k)\cdots\nu_{x_0,y_0}^1(dx_1\times dy_1)\gamma^0(dx_0\times dy_0).$$
Also, recall that there exists a Markovian coupling induced by $\gamma^0$ and these 1-step couplings $\nu_{\bullet,\bullet}^1,\nu_{\bullet,\bullet}^2,\dots,\nu_{\bullet,\bullet}^k$ by employing Kolmogorov extension theorem (cf. \Cref{rmk:MarkCoupExist}). Therefore, again by \Cref{thm:dwlstochastic}, we have $d^{c^k}(\alpha^k,\beta^k)\geq\dWLk\lc(\mX,\ell_X),(\mY,\ell_Y)\rc$. This completes the proof.

\subsection{The WL distance as COT in the label space $Z$}\label{sec: wl in Z}
In \Cref{thm:dwl cot} we are considering a special COT problem where the two involved state spaces for stochastic processes are different. We note that in the literature, most of the time COT is considering stochastic processes on the same state space. In this section, we then provide an alternative characterization of the WL distance in terms of COT on a same state space under certain conditions.

Given a $Z$-LMMC $(\mX,\ell_X)$, if we require that $\ell_X:X\rightarrow Z$ is injective, then $(\mX,\ell_X)$ induces a MMC on $Z$:  
\[m^{\ell_X}_z:=\begin{cases}
    (\ell_X)_\#m_x^X, &\exists x\in X \text{ such that }\ell_X(x)=z\\
    \delta_z, &\text{otherwise}
\end{cases}\]
We further let $\mu_{\ell_X}:=(\ell_X)_\#\mu_X$ (note, however, $\mu_{\ell_X}$ may not be fully supported on $Z$).

Now, given any two $Z$-LMMCs $(\mX,\ell_X)$ and $(\mY,\ell_Y)$ and $k\in \N$, if we assume that $\ell_X$ and $\ell_Y$ are injective, then these two LMMCs induce two Markov chains on the same state space $Z$. 
It is natural to wonder whether one can characterize the WL distance of depth $k$ via a COT problem on $Z$. 
Recall notation $\alpha^k\in\mathcal{P}(X^{k+1})$ and $\beta^k\in\mathcal{P}(Y^{k+1})$ from \Cref{sec:COT}. 
We now let $\alpha_Z^k:=(\ell_X^{k+1})_\#\alpha^k\in \mathcal{P}(Z^{k+1})$ and $\beta_Z^k:=(\ell_Y^{k+1})_\#\beta^k\in \mathcal{P}(Z^{k+1})$. Here $\ell_X^{k+1}:X^{k+1}\rightarrow Z^{k+1}$ denotes the self product of $\ell_X$, and $\ell_Y^{k+1}$ is similarly defined.
Finally, we consider the cost $d_Z^k:Z^{k+1}\times Z^{k+1}\rightarrow \R$ defined by $d_Z^k((z_0,\ldots,z_k),(z_0',\ldots,z_k')):=d_Z(z_k,z_k')$. Note that $d_Z^k\circ(\ell_X^{k+1}\times\ell_Y^{k+1})=c^k$, where $c^k$ is defined through \Cref{ck}.
Then, we have the following result:
\begin{theorem}\label{thm:dwl cot on Z}
For any $Z$-LMMCs $(\mX,\ell_X)$ and $(\mY,\ell_Y)$, if we assume that $\ell_X$ and $\ell_Y$ are injective, then we have that
\[\dWLk\lc(\mX,\ell_X),(\mY,\ell_Y)\rc=d^{d_Z^k}\lc\alpha_Z^k,\beta_Z^k\rc.\]
\end{theorem}

For the proof of \Cref{thm:dwl cot on Z}, we establish the following lemma.

\begin{lemma}\label{lemma:bicausalpwdsurj}
$(\ell_X^{k+1}\times\ell_Y^{k+1})_\#$ induces a surjective map from $\cpl_{\mathrm{bc}}(\alpha^k,\beta^k)$ to $\cpl_{\mathrm{bc}}(\alpha_Z^k,\beta_Z^k)$.
\end{lemma}

Given this lemma, and by the fact that $d_Z^k\circ(\ell_X^{k+1}\times\ell_Y^{k+1})=c^k$, it is easy to check that $d^{c^k}(\alpha^k,\beta^k)=d^{d_Z^k}\lc\alpha_Z^k,\beta_Z^k\rc$. Hence, by employing \Cref{thm:dwl cot}, we conclude our proof of \Cref{thm:dwl cot on Z}. 

Finally, we provide a proof of \Cref{lemma:bicausalpwdsurj}.

\begin{proof}[Proof of \Cref{lemma:bicausalpwdsurj}]
Fix an arbitrary $\pi^k\in\cpl_{\mathrm{bc}}(\alpha^k,\beta^k)$. Let $\pi_Z^k:=(\ell_X^{k+1}\times\ell_Y^{k+1})_\#\pi^k$. Then, it is easy to check that $\pi_Z^k$ is indeed a bicausal coupling between $\alpha_Z^k$ and $\beta_Z^k$, so we omit the proof.

Next, fix an arbitrary $\pi_Z^k\in\cpl_{\mathrm{bc}}(\alpha_Z^k,\beta_Z^k)$. Now, let us define a probability measure $\pi^k$ on $X^{k+1}\times Y^{k+1}$ in the following way:
$$\pi^k((x_0,\dots,x_k),(y_0,\dots,y_k)):=\pi_Z^k((\ell_X(x_0),\dots,\ell_X(x_k)),(\ell_Y(y_0),\dots,\ell_Y(y_k)))$$
for all $((x_0,\dots,x_k),(y_0,\dots,y_k))\in X^{k+1}\times Y^{k+1}$. First of all, observe that 
\[\pi_Z^k((z_0,\dots,z_k),(z_0',\dots,z_k'))=0\]
if $z_i\notin\ell_X(X)$ for some $i\in\{0,\dots,k\}$ or $z_j'\notin\ell_Y(Y)$ for some $j\in\{0,\dots,k\}$ since $\pi_Z^k$ is a coupling measure between $\alpha_Z^k$ and $\beta_Z^k$. This implies that indeed $\pi^k$ is a coupling measure between $\alpha^k$ and $\beta^k$. Furthermore, since $\ell_X$ and $\ell_Y$ are injective, it is easy to check that $\pi_Z^k=(\ell_X^{k+1}\times\ell_Y^{k+1})_\#\pi^k$. Finally, note that
\begin{align*}
    \pi^k((y_0,\dots,y_l)\vert (x_0,\dots,x_k))&=\pi_Z^k((\ell_Y(y_0),\dots,\ell_Y(y_l))\vert (\ell_X(x_0),\dots,\ell_X(x_k)))\\
    &=\pi_Z^k((\ell_Y(y_0),\dots,\ell_Y(y_l))\vert (\ell_X(x_0),\dots,\ell_X(x_l)))\quad(\because\,\pi_Z^k\text{ is bicausal})\\
    &=\pi^k((y_0,\dots,y_l)\vert (x_0,\dots,x_l))
\end{align*}
for all $l\in\{0,\dots,k\}$, $(y_0,\dots,y_l)\in Y^{l+1}$, and $(x_0,\dots,x_k)\in X^{k+1}$. Hence, $\pi^k$ is causal from $\alpha^k$ to $\beta^k$. In a similar way one can also prove that $\pi^k$ is causal from $\beta^k$ to $\alpha^k$, too. This completes the proof.
\end{proof}


\section{Details from \Cref{sec: GNN}}

\subsection{Details about Markov chain neural networks (MCNNs)}\label{sec:MCNN}

We first briefly recall the definition of MCNNs from \citep{chen2022weisfeilerlehman}. 

Given any Lipschitz map $\varphi:\R^i\rightarrow\R^j$, define the map $q_\varphi:\mathcal{P}(\R^i)\rightarrow\R^j$ by sending $\alpha$ to $\int_{\R^i}\varphi(x)\alpha(dx)$. 

\paragraph{One layer of MCNN} We define one layer of MCNN as follows: $F_\varphi:\mathcal{M}^L(\R^i)\rightarrow \mathcal{M}^L(\R^j)$ sends $(\mX,\ell_X:X\rightarrow \R^i)$ to $(\mX,\ell_X^\varphi:X\rightarrow \R^j)$ where $\ell_X^\varphi(x):=q_\varphi((\ell_X)_\#m_x^X)$ for each $x\in X$. 

\paragraph{Readout layer} We define a readout layer for MCNN as follows: we first let $S_\varphi:\mathcal{M}^L(\R^i)\rightarrow\R^j$ be defined via $(\mX,\ell_X)\mapsto q_\varphi((\ell_X)_\#\mu_X)$; then we consider any Lipschitz $\psi:\R^j\rightarrow \R$ and call $\psi\circ S_\varphi:\mathcal{M}^L(\R^i)\rightarrow\R$ a readout layer.

\paragraph{$k$-layer MCNNs}
Now, a $k$-layer MCNN is defined as follows. Given a sequence of MLPs $\varphi_i:\R^{d_{i-1}}\rightarrow \R^{d_i}$ for $i=1,\ldots,k+1$ and a MLP $\psi:\R^{d_{k+1}}\rightarrow\R$, the map of the following form is called a $k$-layer MCNN:
\[H:=\psi\circ S_{\varphi_{k+1}}\circ F_{\varphi_k}\circ\cdots\circ F_{\varphi_1}:\mathcal{M}^L(\R^d)\rightarrow \R.\]
We let $\mathcal{N\!N}_k(\R^d)$ denote the collection of all $k$-layer MCNNs.

\paragraph{MCNNs are MP-GNNs}
Now, we show how MCNNs reduce to MP-GNNs when restricted to graph induced LMMCs. Let $h:\mathcal{G}(\R^d)\rightarrow \R$ be a $k$-layer MP-GNN as defined in \Cref{sec:MPNN}. Let $\varphi_i:\R^{d_{i-1}}\rightarrow \R^{d_i}$ for $i=1,\ldots,k+1$ and  $\psi:\R^{d_{k+1}}\rightarrow\R$ be the maps involved in defining $h$. Then, we consider the following $k$-layer MCNN:
\[H:=\psi\circ S_{\varphi_{k+1}}\circ F_{\varphi_k}\circ\cdots\circ F_{\varphi_1}:\mathcal{M}^L(\R^d)\rightarrow \R.\]
We now carry out calculations for the first layer given a graph induced LMMC as input. Let $(G,\ell_G)\in\mathcal{G}(\R^d)$. Consider the corresponding LMMC $(\mX_q(G),\ell_G)$ defined in \Cref{sec: graph induced LMMC}.
Then, by applying $H$ to $(\mX_q(G),\ell_G)$, one has that for any $v\in V$ (without loss of generality, we assume $\deg(v)>0$), 
\begin{align*}
\ell^{\varphi_1}_G(v) &= q_{\varphi_1}((\ell_G)_\#m_v^{G,q})=q_{\varphi_1}\left(q\delta_{\ell_G(v)}+\frac{1-q}{\deg(v)}\sum_{v'\in N_G(v)}w_{vv'}\delta_{\ell_G(v')}\right)\\
    &= \int_{\R^d}\varphi_1(v'')\left(q\delta_{\ell_G(v)}+\frac{1-q}{\deg(v)}\sum_{v'\in N_G(v)}w_{vv'}\delta_{\ell_G(v')}\right)(dv'')\\
    &=q\varphi_1({\ell_G(v)})+\frac{1-q}{\deg(v)}\sum_{v'\in N_G(v)}w_{vv'}\varphi_1({\ell_G(v')}).
\end{align*}
Hence the label map $\ell_G^{\varphi_1}:V\rightarrow\R^{d_1}$ agrees with the one $\ell_G^1$ obtained via the message passing rule specified in \Cref{sec:MPNN}.
In this way, it is straightforward to verify that any MP-GNN defined in Section \ref{sec:MPNN} is a restriction of an MCNN to the collection of graph induced LMMCs $\mathcal{G}_q(\R^d)$, i.e.,
\begin{equation}\label{eq:h H}
    h=H\circ I_q.
\end{equation}
Here recall that $I_q:\mathcal{G}(\R^d)\rightarrow\mathcal{M}^L(\R^d)$ sends a labeled graph to a LMMC and $I_q$ has image $\mathcal{G}_q(\R^d)$.  

\paragraph{Universal approximation property of MCNNs}

Given any $k\in\N$ and any subset $\mathcal{K}\subseteq \lc\mathcal{G}(\R^d),d_{\mathcal{G},q}^{\scriptscriptstyle{(k)}}\rc$, we let $\mathcal{N\!N}_k^q(\R^d)|_\mathcal{K}$ denote the collection of restrictions of functions $h\in \mathcal{N\!N}_k(\R^d)$ to $\mathcal{K}$, i.e.,
\[\mathcal{N\!N}_k(\R^d)|_\mathcal{K}:=\left\{h|_\mathcal{K}:\,h\in \mathcal{N\!N}_k(\R^d)\right\}.\]

The following theorem is a slight variant of \citep[Theorem 4.3]{chen2022weisfeilerlehman}: Notice that $\dWLk$ is a pseudo-distance on $\mathcal{M}^L(\R^d)$. Hence in \citep[Theorem 4.3]{chen2022weisfeilerlehman}, all pseudometric spaces are first transformed to metric spaces by identifying points at 0 distance (cf. \citep[Proposition 1.1.5]{burago2001course}). 
In the following result, we remove this subtlety and deal with the pseudo-distance topology directly.
\begin{theorem}\label{thm: universal MCNN pseudo}
For any $k\in\N$ and any compact subset $\mathcal{K}\subseteq\left(\mathcal{M}^L(\R^d),\dWLk\right)$, one has that $\overline{\mathcal{N\!N}_k(\R^d)|_\mathcal{K}}=C(\mathcal{K},\mathbb{R})$.
\end{theorem}
\begin{proof}
We let $\mathcal{M}^L_k(\R^d)$ denote the space obtained by identifying points at 0 distance {(w.r.t $\dWLk$)} in $\mathcal{M}^L(\R^d)$. Then, $\dWLk$ become a metric on $\mathcal{M}^L_k(\R^d)$. Let $Q$ denote this identification map (also called the \emph{quotient} map):
    \begin{equation}\label{eq:pseudo}
        Q:\left(\mathcal{M}^L(\R^d),\dWLk\right)\rightarrow \left(\mathcal{M}^L_k(\R^d),\dWLk\right).
    \end{equation}

Let $\mathcal{K}_Q:=Q(\mathcal{K})$. Then, $\mathcal{K}_Q$ is compact in $\mathcal{M}^L_k(\R^d)$. Consider the map $Q_c:C(\mathcal{K}_Q,\R)\rightarrow C(\mathcal{K},\R)$ defined by sending $f_Q$ to $f_Q\circ Q$. It is easy to check that $Q_c$ is an isometric embedding w.r.t. the sup metrics on $C(\mathcal{K}_Q,\R)$ and $C(\mathcal{K},\R)$. It is easy to check that $Q_c$ has its isometric inverse $(Q_c)^{-1}:C(\mathcal{K},\R)\rightarrow C(\mathcal{K}_Q,\R)$ defined as follows, which is also an isometric embedding: for any $f\in C(\mathcal{K},\R)$, we let $f_Q:\mathcal{K}_Q\rightarrow\R$ be defined such that for any equivalence class $[(\mX,\ell_X)]\in\mathcal{K}_Q$, $f_Q([(\mX,\ell_X)]):=f((\mX,\ell_X))$ (it is obvious that $f_Q$ does not depend on the choice of the representative $(\mX,\ell_X)$.). Hence, $Q_c$ gives rise to an isometry from $C(\mathcal{K}_Q,\R)$ to $C(\mathcal{K},\R)$.

Now, by \citep[Theorem 4.3]{chen2022weisfeilerlehman}, one has that
\[\overline{Q_c\left(\NN_k(\R^d)|_\mathcal{K}\right)}=C(\mathcal{K}_Q,\R).\]

Since $Q_c$ is an isometry, we have that
$\overline{\mathcal{N\!N}_k(\R^d)|_\mathcal{K}}=C(\mathcal{K},\mathbb{R})$.  
\end{proof}

\subsection{Proof of \Cref{thm: better than gin}}
The proof of the theorem is based on the following several lemmas. 
\begin{lemma}\label{lm: countable q}
We let $P_S:=\{\sum_{i=1}^np_i:p_1,\ldots,p_n\in P\}$ denote the collection of finite sums of elements in $P$. Then, for a fixed $q\in (0,1)$, let 
    \begin{align*}
        Q:=&\{1\}\cup\left\{\frac{\overline{\deg}(v)}{\sum_{v'\in V}\overline{\deg}(v')}:G\text{ has edge weights in }P, \, v\in V\right\}\\
        \cup& \left\{mq+s\frac{(1-q)}{\deg(v)}:G\text{ has edge weights in }P, \, v\in V \text{ is such that }\deg(v)>0, m\in \N,\,s\in P_S\right\}.
    \end{align*}
    Then, $Q$ is a countable set.
\end{lemma}
\begin{proof}
Note that $P_S$ is still a countable set.
Since graphs have edge weights from a countable set $P$, every $\deg(v)$ (and $\overline{\deg}(v)$) encountered in the definition above is from a countable set. Hence, $Q$ is countable.    
\end{proof}

\begin{lemma}
Consider the union
$$C_{P}(Z)\coloneqq\bigcup_{(G,\ell_G)\in\mathcal{G}_P(Z)}\lc\{(\ell_G)_\#\mu_G\}\cup\left\{(\ell_G)_\#m_v^{G,q}:\,v\in V\right\}\rc\subseteq \mathcal{P}(\R^d)$$
Then, $C_{P}(Z)$ is a countable subset of $\mathcal{P}(\R^d)$.
\end{lemma}
\begin{proof}
Every probability measure $\alpha\in C_P(Z)$ can be expressed as
\[\alpha=\sum_{z\in Z}\alpha(z)\delta_z, \text{ where }\alpha(z)\text{ is a sum of finitely many elements in }Q.\]

Then, by the fact that both $Z$ and $Q$ are countable (cf. \Cref{lm: countable q}), we have that $C_P(Z)$ is countable.
\end{proof}

Recall notation from \Cref{sec:MCNN}. Then, we have that
\begin{lemma}\label{lm:countable image}
Fix any $d\in\mathbb{N}$. Then, for any Lipschitz $\varphi:\R^d\rightarrow\R$, any countable subset $Z\subseteq\R^d$ and any countable subset $P\subseteq \R$, there exists a countable subset $Z'\subseteq\R$ such that 
\begin{equation}\label{eq:countable q}
    F_{\varphi}(\mathcal{G}_P(Z))\subseteq \mathcal{G}_{P}(Z')\text{ and }S_{\varphi}(\mathcal{G}_P(Z))\subseteq Z'.
\end{equation}
\end{lemma}

\begin{proof}
Let $Z'\coloneqq q_\varphi(C_P(Z))\subseteq\R$. Since $C_P(Z)$ is countable, $Z'$ is also countable. It is obvious that $Z'$ satisfies \Cref{eq:countable q} which concludes the proof.
\end{proof}

\begin{lemma}\label{lm:restriction injective}
For any countable subset $C\subseteq\mathcal{P}(\R^d)$, there exists a Lipschitz function $\varphi:\R^d\rightarrow\R$ such that the restriction of $q_\varphi:\mathcal{P}(\R^d)\rightarrow\R$ to $C$ is injective.
\end{lemma}
\begin{proof}
Consider the following vector space:
$$\mathrm{Lip}_0(\R^d)\coloneqq\{f:\R^d\rightarrow\R|\,f\text{ is Lipschitz and }f(0)=0\}.$$
When equipped with the norm $\|f\|\coloneqq\sup_{x\neq y}\frac{|f(x)-f(y)|}{\|x-y\|}$, $\mathrm{Lip}_0(\R^d)$ is a Banach space \citep{weaver1995order}. 
Consider the following collection of signed measures $D\coloneqq\{\alpha-\beta:\,\alpha,\beta\in C\text{ and }\alpha\neq \beta\}$. 
Since $C$ is countable, $D$ is also countable. 
Each $\mu\in D$ gives rise to a linear functional $\psi_\mu:\mathrm{Lip}_0(\R^d)\rightarrow\R$ defined by $f\mapsto\int_{\R^d}f(x)\mu(dx)$. 
This functional is well-defined since every probability measure in $\mathcal{P}(\R^d)$ is assumed to have finite 1-moment (see \Cref{sec:probability} for this assumption on $\mathcal{P}(\R^d)$). Moreover, it is easy to check that $\psi_\mu$ is bounded and non-zero. Then, $\ker(\psi_\mu)$ is nowhere dense since it is a proper closed subspace of $\mathrm{Lip}_0(\R^d)$. By Baire category theorem, we have that $\cup_{\mu\in D}\ker(\psi_\mu)\neq\mathrm{Lip}_0(\R^d)$. Then, choose an arbitrary $\varphi\in \mathrm{Lip}_0(\R^d)\backslash\cup_{\mu\in D}\ker(\psi_\mu)$ and it follows that $q_\varphi|_C$ is injective.
\end{proof}

\begin{proof}[Proof of \Cref{thm: better than gin}]
We let $Z_0\coloneqq Z$ and $C_1\coloneqq C_P(Z_0)$. By \Cref{lm:restriction injective}, there exists a Lipschitz map $\varphi_1:\R^d\rightarrow\R$ such that $q_{\varphi_1}|_{C_1}$ is injective. By \Cref{lm:countable image}, there exists a countable subset $Z_1\subseteq\R$ such that 
$$F_{\varphi_1}(\mathcal{G}_{P}(Z_0))\subseteq\mathcal{G}_{P}(Z_1). $$
Let $C_2\coloneqq C_{P}(Z_1)$. By \Cref{lm:restriction injective} again, there exists a Lipschitz map $\varphi_2:\R\rightarrow\R$ such that $q_{\varphi_2}|_{C_2}$ is injective. Then, inductively, for each $i=1,\ldots,k$, there exist a countable subset $Z_i\subseteq\R$ and a Lipschitz map $\varphi_i:\R\rightarrow\R$ so that $q_{\varphi_i}$ is injective when restricted to $C_i\coloneqq C_{P}(Z_{i-1})$ and
$F_{\varphi_{i}}(\mathcal{G}_{P}(Z_{i-1}))\subseteq\mathcal{G}_{P}(Z_i)$. Similarly, there exist a countable subset $S\subseteq\R$ and a Lipschitz map $\varphi_{k+1}:\R\rightarrow\R$ such that $q_{\varphi_{k+1}}$ is injective when restricted to $C_{k+1}\coloneqq C_{P}(Z_{k})$ and
$S_{\varphi_{k+1}}(\mathcal{G}_{P}(Z_{k}))\subseteq S$. We then let $\psi:\R\rightarrow\R$ denote the identity map and let $h\coloneqq\psi\circ S_{\varphi_{k+1}}\circ F_{\varphi_k}\circ \cdots\circ F_{\varphi_1}:\mathcal{G}(\R^d)\rightarrow \mathbb{R}.$

Pick any $(G_1,\ell_{G_1}),(G_2,\ell_{G_2})\in\mathcal{G}_P(Z)$. We prove that $\dGk((G_1,\ell_{G_1}),(G_2,\ell_{G_2}))>0$ iff $h((G_1,\ell_{G_1}))\neq h((G_2,\ell_{G_2}))$.

For each $i=1,\ldots,k$, we let
\begin{equation}\label{eq:notation ell lip}
    \lc G_1,\ell_{G_1}^{\scriptscriptstyle{(\varphi,i)}}\rc \coloneqq F_{\varphi_i}\circ\cdots\circ F_{\varphi_1}((G_1,\ell_{G_1})).
\end{equation}
We similarly define $\lc G_2,\ell_{G_2}^{\scriptscriptstyle{(\varphi,i)}}\rc$. 
Then, we prove that for each $i=1,\ldots,k$, 
\begin{equation}\label{eq:lip separate points}
    \forall v_1\in V_{G_1},\,v_2\in V_{G_2},\,\, \ell_{G_1}^{\scriptscriptstyle{(\varphi,i)}}(v_1)=\ell_{G_2}^{\scriptscriptstyle{(\varphi,i)}}(v_2)\mbox{ iff }\WLh{i}{(\mX_q(G_1),\ell_{G_1})}(v_1)=\WLh{i}{(\mX_q(G_2),\ell_{G_2})}(v_2)
\end{equation}

Given \Cref{eq:lip separate points}, it is obvious that 
\[\lc \WLh{k}{(\mX_q(G_1),\ell_{G_1})}\rc_\#\mu_{G_1}\neq \lc \WLh{k}{(\mX_q(G_2),\ell_{G_2})}\rc_\#\mu_{G_2} \text{ iff } \lc \ell_{G_1}^{\scriptscriptstyle{(\varphi,k)}}\rc_\#\mu_{G_1}\neq \lc \ell_{G_2}^{\scriptscriptstyle{(\varphi,k)}}\rc_\#\mu_{G_2}.\]
Note that $\lc \ell_{G_1}^{\scriptscriptstyle{(\varphi,k)}}\rc_\#\mu_{G_1},\lc \ell_{G_2}^{\scriptscriptstyle{(\varphi,k)}}\rc_\#\mu_{G_2}\in C_{k+1}$. Since $q_{\varphi_{k+1}}$ is injective on $C_{k+1}$ and $\psi$ is injective on $S$, we have that $\lc \ell_{G_1}^{\scriptscriptstyle{(\varphi,k)}}\rc_\#\mu_{G_1}\neq \lc \ell_{G_2}^{\scriptscriptstyle{(\varphi,k)}}\rc_\#\mu_{G_2}$ iff $h((G_1,\ell_{G_1}))\neq h((G_2,\ell_{G_2}))$. Therefore, $\dGk((G_1,\ell_{G_1}),(G_2,\ell_{G_2}))>0$ iff $h((G_1,\ell_{G_1}))\neq h((G_2,\ell_{G_2}))$.

To conclude the proof, we prove \Cref{eq:lip separate points} by induction on $i=1,\ldots,k$. When $i=1$, since $(\ell_{G_1})_\#m_{v_1}^{G_1,q},(\ell_{G_2})_\#m_{v_2}^{G_2,q}\in C_1,\, \forall v_1\in V_{G_1},\,v_2\in V_{G_2}$, by injectivity of $q_{\varphi_1}$ on $C_1$, we have that 
\[ \forall v_1\in V_{G_1},\,v_2\in V_{G_2},\,\, (\ell_{G_1})_\#m_{v_1}^{G_1,q}=(\ell_{G_2})_\#m_{v_2}^{G_2,q}\mbox{ iff }q_{\varphi_1}\big((\ell_{G_1})_\#m_{v_1}^{G_1,q}\big)= q_{\varphi_1}\big((\ell_{G_2})_\#m_{v_2}^{G_2,q}\big).\]
Equivalent speaking, 
\[ \forall v_1\in V_{G_1},\,v_2\in V_{G_2},\,\, \WLh{1}{(\mX_q(G_1),\ell_{G_1})}(v_1)=\WLh{1}{(\mX_q(G_2),\ell_{G_2})}(v_2)\mbox{ iff }\ell_{G_1}^{\scriptscriptstyle{(\varphi,1)}}(x)=\ell_{G_2}^{\scriptscriptstyle{(\varphi,1)}}(v_2).\]

Now, we assume that \Cref{eq:lip separate points} holds for some $i\geq 1$. For $i+1$, since 
$$\lc \ell_{G_1}^{\scriptscriptstyle{(\varphi,i)}}\rc_\#m_{v_1}^{G_1,q},\lc \ell_{G_2}^{\scriptscriptstyle{(\varphi,i)}}\rc_\#m_{v_2}^{G_2,q}\in C_{i+1},\, \forall v_1\in V_{G_1},\,v_2\in V_{G_2},$$
by injectivity of $q_{\varphi_{i+1}}$ on $C_{i+1}$, we have that $\forall v_1\in V_{G_1},v_2\in V_{G_2},$
\[\lc \ell_{G_1}^{\scriptscriptstyle{(\varphi,i)}}\rc_\#m_{v_1}^{G_1,q}=\lc \ell_{G_2}^{\scriptscriptstyle{(\varphi,i)}}\rc_\#m_{v_2}^{G_2,q}\mbox{ iff }q_{\varphi_{i+1}}\lc\lc \ell_{G_1}^{\scriptscriptstyle{(\varphi,i)}}\rc_\#m_{v_1}^{G_1,q}\rc= q_{\varphi_{i+1}}\lc\lc \ell_{G_2}^{\scriptscriptstyle{(\varphi,i)}}\rc_\#m_{v_2}^{G_2,q}\rc.\]
Equivalent speaking, 
\[ \forall v_1\in V_{G_1},\,y\in Y,\,\, \lc \ell_{G_1}^{\scriptscriptstyle{(\varphi,i)}}\rc_\#m_{v_1}^{G_1,q}=\lc \ell_{G_2}^{\scriptscriptstyle{(\varphi,i)}}\rc_\#m_{v_2}^{G_2,q}\mbox{ iff }\ell_{G_1}^{\scriptscriptstyle{(\varphi,i+1)}}(x)=\ell_{G_2}^{\scriptscriptstyle{(\varphi,i+1)}}(y).\]
By the induction assumption, $\forall v_1\in V_{G_1},\,v_2\in V_2$ we have that
\[ \ell_{G_1}^{\scriptscriptstyle{(\varphi,i)}}(v_1)=\ell_{G_2}^{\scriptscriptstyle{(\varphi,i)}}(v_2)\mbox{ iff }\WLh{i}{(\mX_q(G_1),\ell_{G_1})}(v_1)=\WLh{i}{(\mX_q(G_2),\ell_{G_2})}(v_2).\]
This implies that
\begin{align*}
    &\lc \ell_{G_1}^{\scriptscriptstyle{(\varphi,i)}}\rc_\#m_{v_1}^{G_1,q}=\lc \ell_{G_2}^{\scriptscriptstyle{(\varphi,i)}}\rc_\#m_{v_2}^{G_2,q}\\
    \mbox{ iff }&\lc\WLh{i}{(\mX_q(G_1),\ell_{G_1})}\rc_\#m_{v_1}^{G_1,q}=\lc\WLh{i}{(\mX_q(G_2),\ell_{G_2})}\rc_\#m_{v_2}^{G_2,q}\\
    \mbox{ iff }&\WLh{i+1}{(\mX_q(G_1),\ell_{G_1})}(v_1)=\WLh{i+1}{(\mX_q(G_2),\ell_{G_2})}(v_2).
\end{align*}
\[ \]
Therefore, 
\[ \ell_{G_1}^{\scriptscriptstyle{(\varphi,i+1)}}(v_1)=\ell_{G_2}^{\scriptscriptstyle{(\varphi,i+1)}}(v_2)\mbox{ iff }\WLh{i+1}{(\mX_q(G_1),\ell_{G_1})}(v_1)=\WLh{i+1}{(\mX_q(G_2),\ell_{G_2})}(v_2)\]
and we thus conclude the proof.
\end{proof}

\subsection{Proof of \Cref{thm: lip}}
We need the following basic fact.
\begin{lemma}[{\citep[Lemma B.4]{chen2022weisfeilerlehman}}]\label{lm:lipschitz inherit}
For any $C$-Lipschitz function $\varphi:\R^i\rightarrow\R^j$, we have that the map $q_\varphi:\prob(\R^i)\rightarrow\R^j$ is $C$-Lipschitz.
\end{lemma}

Recall notations from \Cref{sec:MCNN}. Given a $k$-layer MCNN $H\coloneqq\psi\circ S_{\varphi_{k+1}}\circ F_{\varphi_k}\circ \cdots \circ F_{\varphi_1}$, recall from \Cref{eq:h H} that $h:=H\circ I_q$ gives rise to a MP-GNN.
For any $(G,\ell_G)\in\mathcal{G}(Z)$, consider its induced LMMC $(\mX_q(G),\ell_G)$.
Following notation in \Cref{eq:notation ell lip} we let
\begin{equation}\label{eq:notation ell}
    \lc \mX_q(G),\ell_G^{\scriptscriptstyle{(\varphi,i)}}\rc \coloneqq F_{\varphi_i}\circ\cdots\circ F_{\varphi_1}((\mX_q(G),\ell_G))
\end{equation}

Now, we assume that for $i=1,\ldots,k$, $\varphi_i$ is a $C_i$-Lipschitz MLP for some $C_i>0$. Then, by \Cref{lm:lipschitz inherit}, we have that $q_{\varphi_i}$ is a $C_i$-Lipschitz map for $i=1,\ldots,k$.

Then, we prove that
\begin{equation}\label{eq:lip ineq}
    \dW\!\lc \lc \ell_{G_1}^{\scriptscriptstyle{(\varphi,k)}}\rc_\#\mu_{G_1},\lc \ell_{G_2}^{\scriptscriptstyle{(\varphi,k)}}\rc_\#\mu_{G_2}\rc \leq
\Pi_{i=1}^k C_i\cdot\dGk((G_1,\ell_{G_1}),(G_2,\ell_{G_2})).
\end{equation}

Given Equation \eqref{eq:lip ineq}, by \Cref{lm:lipschitz inherit} again and the fact that $\psi$ is $C$-Lipschitz, one has that
\begin{align*}
    |h((G_1,\ell_{G_1}))-h((G_2,\ell_{G_2}))|&\leq C\cdot \left\|q_{\varphi_{k+1}}\lc\lc \ell_{G_1}^{\scriptscriptstyle{(\varphi,k)}}\rc_\#\mu_{G_1}\rc-q_{\varphi_{k+1}}\lc\lc \ell_{G_2}^{\scriptscriptstyle{(\varphi,k)}}\rc_\#\mu_{G_2}\rc\right\|\\
    &\leq C\cdot C_{k+1}\dW\!\lc \lc \ell_{G_1}^{\scriptscriptstyle{(\varphi,k)}}\rc_\#\mu_{G_1},\lc \ell_{G_2}^{\scriptscriptstyle{(\varphi,k)}}\rc_\#\mu_{G_2}\rc \\
    &\leq C\cdot
\Pi_{i=1}^{k+1} C_i\cdot\dGk((G_1,\ell_{G_1}),(G_2,\ell_{G_2})).
\end{align*}

The following proof of Equation \eqref{eq:lip ineq} is adapted from the proof of \citep[Equation (19)]{chen2022weisfeilerlehman}: it suffices to prove that for any $v_1\in V_{G_1}$ and $v_2\in V_{G_2}$, 
\begin{equation*}\label{eq:nn<dwl-k1}
   \left\|\ell_{G_1}^{\scriptscriptstyle{(\varphi,k)}}(v_1)-\ell_{G_2}^{\scriptscriptstyle{(\varphi,k)}}(v_2)\right\|\leq \Pi_{i=1}^k C_i\cdot \dW\!\lc \WLh{k}{(\mX_q(G_1),\ell_{G_1})}(v_1),\WLh{k}{(\mX_q(G_2),\ell_{G_2})}(v_2)\rc.
\end{equation*}
We prove the above inequality by proving the following inequality inductively on $j=1,\ldots,k$:
\begin{equation}\label{eq:nn<dwl1}
   \left\|\ell_{G_1}^{\scriptscriptstyle{(\varphi,j)}}(v_1)-\ell_{G_2}^{\scriptscriptstyle{(\varphi,j)}}(v_2)\right\|\leq \Pi_{i=1}^j C_i\cdot \dW\!\lc \WLh{j}{(\mX_q(G_1),\ell_{G_1})}(v_1),\WLh{j}{(\mX_q(G_2),\ell_{G_2})}(v_2)\rc.
\end{equation}

When $j=1$, we have that
\begin{align*}   \left\|\ell_{G_1}^{\scriptscriptstyle{(\varphi,1)}}(v_1)-\ell_{G_2}^{\scriptscriptstyle{(\varphi,1)}}(v_2)\right\| & = \left\|q_{\varphi_{1}}\!\lc \lc \ell_{G_1}\rc_\#m_{v_1}^{G_1,q}\rc -q_{\varphi_{j}}\!\lc \lc \ell_{G_2}\rc_\#m_{v_2}^{G_2,q}\rc \right\|\\
&\leq C_1\dW\lc \lc \ell_{G_1}\rc_\#m_{v_1}^{G_1,q},\lc \ell_{G_2}\rc_\#m_{v_2}^{G_2,q}\rc\\
&=C_1\dW\!\lc \WLh{1}{(\mX_q(G_1),\ell_{G_1})}(v_1),\WLh{1}{(\mX_q(G_2),\ell_{G_2})}(v_2)\rc.
\end{align*}

We now assume that \Cref{eq:nn<dwl1} holds for some $j\geq 1$. For $j+1$, we have that
\begin{align*}
    &\Pi_{i=1}^{j+1} C_i\cdot\dW\!\lc \WLh{j+1}{(\mX_q(G_1),\ell_{G_1})}(v_1),\WLh{j+1}{(\mX_q(G_2),\ell_{G_2})}(v_2)\rc\\ &=\Pi_{i=1}^{j+1} C_i\cdot\dW\!\lc \lc \WLh{j}{(\mX_q(G_1),\ell_{G_1})}\rc_\#m_{v_1}^{G_1,q},\lc \WLh{j}{(\mX_q(G_2),\ell_{G_2})}\rc_\#m_{v_2}^{G_2,q}\rc \\
    &=C_{j+1}\cdot\inf_{\gamma\in\mathcal{C}\left(m_{v_1}^{G_1,q},m_{v_2}^{G_2,q}\right)}\int\limits_{V_{G_1}\times V_{G_2}}\Pi_{i=1}^j C_i\cdot\dW\!\lc \WLh{j}{(\mX_q(G_1),\ell_{G_1})}(v_1'),\WLh{j}{(\mX_q(G_2),\ell_{G_2})}(v_2')\rc \gamma(dv_1'\times dv_2')\\
    &\geq C_{j+1} \inf_{\gamma\in\mathcal{C}\left(m_{v_1}^{G_1,q},m_{v_2}^{G_2,q}\right)}\int\limits_{V_{G_1}\times V_{G_2}}\left\|\ell_{G_1}^{\scriptscriptstyle{(\varphi,j)}}(v_1')-\ell_{G_2}^{\scriptscriptstyle{(\varphi,j)}}(v_2')\right\|\gamma(dv_1'\times dv_2')\\
    &=C_{j+1}\cdot\dW\!\lc \lc \ell_{G_1}^{\scriptscriptstyle{(\varphi,j)}}\rc_\#m_{v_1}^{G_1,q},\lc \ell_{G_2}^{\scriptscriptstyle{(\varphi,j)}}\rc_\#m_{v_2}^{G_2,q}\rc \\
    &\geq \left\|q_{\varphi_{j+1}}\!\lc \lc \ell_{G_1}^{\scriptscriptstyle{(\varphi,j)}}\rc_\#m_{v_1}^{G_1,q}\rc -q_{\varphi_{j+1}}\!\lc \lc \ell_{G_2}^{\scriptscriptstyle{(\varphi,j)}}\rc_\#m_{v_2}^{G_2,q}\rc \right\|\\
    &=\left\|\ell_{G_1}^{\scriptscriptstyle{(\varphi,j+1)}}(v_1)-\ell_{G_2}^{\scriptscriptstyle{(\varphi,j+1)}}(v_2)\right\|.
\end{align*}

\subsection{More on Lipschitz properties}\label{sec: lip}
In \Cref{sec: graph induced LMMC}, we introduced one way of inducing LMMCs from weighted graphs. Below, we introduce a new method for doing so and hence establish the Lipschitz property of a new type of MP-GNNs closely related to GINs.

Given any $\eps\geq 0$ and any finite edge weighted graph $G=(V,E, w)$ endowed with a label function $\ell_G:V\rightarrow Z$, we generate a LMMC as follows. We associate to the vertex set $V$ a Markov kernel $m_\bullet^{G,(\eps)}$ as follows: for any $v\in V$,
\[m_v^{G,(\eps)}:=\frac{1}{\deg(v)+1+\eps}\lc (1+\eps)\delta_v+\sum_{v'\in N_G(v)}w_{vv'}\delta_{v'}\rc\]
For each vertex $v\in V$, we let $\deg^\eps(v):=\deg(v)+1+\eps$. Then, the following probability measure $\mu_G^\eps:=\sum_{v\in V}\frac{{\deg^\eps}(v)}{\sum_{v'\in V}{\deg^\eps}(v')}\delta_v$ is a stationary distribution w.r.t. the Markov kernel $m_\bullet^{G,(\eps)}$. Then, we let $\mX_{(\eps)}(G):=(V,m_\bullet^{G,(\eps)},\mu_G^\eps)$.

Now, recall that $\mathcal{G}(Z)$ denotes the collection of all $Z$-labeled graphs. Given $\eps\geq 0$, let $I_{(\eps)}:\mathcal{G}(Z)\rightarrow \mathcal{M}^L(Z)$ denote the map sending a $Z$-labeled graph $G$ into a $Z$-LMMC $(\mX_{(\eps)}(G),\ell_G)$.
Let $\mathcal{G}_{(\eps)}(Z):=I_{(\eps)}(\mathcal{G}(Z))\subseteq\mathcal{M}^L(Z)$. 
Now, for any $k\geq 0$, $\dWLk$ restricted on {$\mathcal{G}_{(\eps)}(Z)$} induces a pseudo-distance, which we denote by $d_{\mathcal{G},{(\eps)}}^{\scriptscriptstyle{(k)}}$, on $\mathcal{G}(Z)$. 
Note the subtle differences in notation introduced in \Cref{sec: graph induced LMMC}.

Now, given any $\eps\geq 0$, we consider the following $k$-layer ``normalized'' GIN.
\begin{align*}
 \text{Message Passing:}&\quad   \ell_G^{i+1}(v)=\varphi_{i+1}\lc\frac{1}{\deg(v)+1+\eps}\lc (1+\eps)\ell^i_G(v)+\sum_{v'\in N_G(v)}w_{vv'}\ell^i_G(v')\rc\rc\\
 \text{Readout:}&\quad   h((G,\ell_G)):=\psi\lc \sum_{v\in V}\frac{{\deg^\eps}(v)}{\sum_{v'\in V} {\deg^\eps}(v')}\ell^k_G(v)\rc
\end{align*}
where $\varphi_i:\R^{d_{i-1}}\rightarrow\R^{d_i}$ and $\psi:\R^{d_{k}}\rightarrow \R$ are MLPs.
Note that the only difference between the above MP-GNN and GIN is the involvement of normalization terms $\frac{1}{\deg(v)+1+\eps}$ and $\frac{{\deg^\eps}(v)}{\sum_{v'\in V} {\deg^\eps}(v')}$. 

\paragraph{Lipschitz property of normalized GINs} Now, we establish that the normalized GINs defined above are Lipschitz w.r.t. $d_{\mathcal{G},{(\eps)}}^{\scriptscriptstyle{(k)}}$.
\begin{theorem}
Given a $k$-layer normalized GIN $h:\mathcal{G}(\R^d)\rightarrow\R$ as described above, assume that $\varphi_i$ is $C_i$-Lipschitz for $i=1,\ldots,k$ and that $\psi$ is $C$-Lipschitz.  
Then, for any two labeled graphs $(G_1,\ell_{G_1})$ and $(G_2,\ell_{G_2})$, one has that
\[|h((G_1,\ell_{G_1}))-h((G_2,\ell_{G_2}))|\leq C\cdot\Pi_{i=1}^{k}C_i\cdot d_{\mathcal{G},{(\eps)}}^{\scriptscriptstyle{(k)}}\!\lc(G_1,\ell_{G_1}),(G_2,\ell_{G_2})\rc.\]
\end{theorem}
\begin{proof}
The proof is similar to the one for \Cref{thm: lip}.

We first prove that
\begin{equation}\label{eq:lip ineq eps}
    \dW\!\lc \lc \ell_{G_1}^{k}\rc_\#\mu_{G_1}^\eps,\lc \ell_{G_2}^{k}\rc_\#\mu_{G_2}^\eps\rc \leq
\Pi_{i=1}^k C_i\cdot\dGek((G_1,\ell_{G_1}),(G_2,\ell_{G_2})).
\end{equation}

Given Equation \eqref{eq:lip ineq eps} and the fact that $\psi$ is $C$-Lipschitz, one has that
\begin{align*}
    |h((G_1,\ell_{G_1}))-h((G_2,\ell_{G_2}))|& \leq C\cdot\left\Vert\sum_{v\in V}\frac{{\deg^\eps}(v)}{\sum_{v'\in V} {\deg^\eps}(v')}\ell^k_{G_1}(v)-\sum_{v\in V}\frac{{\deg^\eps}(v)}{\sum_{v'\in V} {\deg^\eps}(v')}\ell^k_{G_2}(v)\right\Vert\\
    &\leq C\cdot \dW\!\lc \lc \ell_{G_1}^{k}\rc_\#\mu_{G_1}^\eps,\lc \ell_{G_2}^{k}\rc_\#\mu_{G_2}^\eps\rc \\
    &\leq C\cdot
\Pi_{i=1}^{k} C_i\cdot\dGek((G_1,\ell_{G_1}),(G_2,\ell_{G_2})),
\end{align*}
where the second inequality follows \Cref{lm:lipschitz inherit} by letting $\varphi=\mathrm{id}_{\R^{d_{k}}}$.

To prove Equation \eqref{eq:lip ineq eps}, it suffices to prove that for any $v_1\in V_{G_1}$ and $v_2\in V_{G_2}$, 
\begin{equation*}\label{eq:nn<dwl-k}
   \left\|\ell_{G_1}^{k}(v_1)-\ell_{G_2}^{k}(v_2)\right\|\leq \Pi_{i=1}^k C_i\cdot \dW\!\lc \WLh{k}{(\mX_{(\eps)}(G_1),\ell_{G_1})}(v_1),\WLh{k}{(\mX_{(\eps)}(G_2),\ell_{G_2})}(v_2)\rc.
\end{equation*}
We prove the above inequality by proving the following inequality inductively on $j=1,\ldots,k$:
\begin{equation}\label{eq:nn<dwl}
   \left\|\ell_{G_1}^{j}(v_1)-\ell_{G_2}^{j}(v_2)\right\|\leq \Pi_{i=1}^j C_i\cdot \dW\!\lc \WLh{j}{(\mX_{(\eps)}(G_1),\ell_{G_1})}(v_1),\WLh{j}{(\mX_{(\eps)}(G_2),\ell_{G_2})}(v_2)\rc.
\end{equation}

When $j=1$, we have that
\begin{align*}   &\left\|\ell_{G_1}^1(v_1)-\ell_{G_2}^1(v_2)\right\| \\
& \leq C_1\left\|\frac{(1+\eps)\ell_{G_1}(v_1)+\sum_{v_1'\in N_{G_1}(v_1)}w_{v_1v_1'}\ell_{G_1}(v_1')}{\deg^\eps_{G_1}(v_1)}-\frac{(1+\eps)\ell_{G_2}(v_2)+\sum_{v_2'\in N_{G_2}(v_2)}w_{v_2v_2'}\ell_{G_2}(v_2')}{\deg^\eps_{G_2}(v_2)}\right\|\\
&=C_1\left\|q_{\mathrm{Id}}\lc (\ell_{G_1})_\#m_{v_1}^{G_1,(\eps)}\rc-q_{\mathrm{Id}}\lc (\ell_{G_2})_\#m_{v_1}^{G_2,(\eps)}\rc\right\|\\
&\leq C_1 \dW\lc (\ell_{G_1})_\#m_{v_1}^{G_1,(\eps)},(\ell_{G_2})_\#m_{v_1}^{G_2,(\eps)}\rc=C_1\dW\!\lc \WLh{1}{(\mX_{(\eps)}(G_1),\ell_{G_1})}(v_1),\WLh{1}{(\mX_{(\eps)}(G_2),\ell_{G_2})}(v_2)\rc.
\end{align*}
Here $\mathrm{Id}:\R^{d_1}\rightarrow \R^{d_1}$ is the identity map and thus is $1$-Lipschitz.

We now assume that \Cref{eq:nn<dwl} holds for some $j\geq 1$. For $j+1$, we have that
\begin{align*}   &\left\|\ell_{G_1}^{j+1}(v_1)-\ell_{G_2}^{j+1}(v_2)\right\| \\
& \leq C_{j+1}\left\|\frac{(1+\eps)\ell_{G_1}^j(v_1)+\sum_{v_1'\in N_{G_1}(v_1)}w_{v_1v_1'}\ell_{G_1}^j(v_1')}{\deg^\eps_{G_1}(v_1)}-\frac{(1+\eps)\ell_{G_2}^j(v_2)+\sum_{v_2'\in N_{G_2}(v_2)}w_{v_2v_2'}\ell_{G_2}^j(v_2')}{\deg^\eps_{G_2}(v_2)}\right\|\\
&=C_{j+1}\left\|q_{\mathrm{Id}}\lc (\ell_{G_1}^j)_\#m_{v_1}^{G_1,(\eps)}\rc-q_{\mathrm{Id}}\lc (\ell_{G_2}^j)_\#m_{v_1}^{G_2,(\eps)}\rc\right\|\\
&\leq C_{j+1} \dW\lc (\ell_{G_1}^j)_\#m_{v_1}^{G_1,(\eps)},(\ell_{G_2}^j)_\#m_{v_1}^{G_2,(\eps)}\rc\\
    &= C_{j+1} \inf_{\gamma\in\mathcal{C}\left(m_{v_1}^{G_1,(\eps)},m_{v_2}^{G_2,(\eps)}\right)}\int\limits_{V_{G_1}\times V_{G_2}}\left\|\ell_{G_1}^{j}(v_1')-\ell_{G_2}^{j}(v_2')\right\|\gamma(dv_1'\times dv_2')\\
&\leq C_{j+1}\cdot\inf_{\gamma\in\mathcal{C}\left(m_{v_1}^{G_1,(\eps)},m_{v_2}^{G_2,(\eps)}\right)}\int\limits_{V_{G_1}\times V_{G_2}}\Pi_{i=1}^j C_i\cdot\dW\!\lc \WLh{j}{(\mX_{(\eps)}(G_1),\ell_{G_1})}(v_1'),\WLh{j}{(\mX_{(\eps)}(G_2),\ell_{G_2})}(v_2')\rc \gamma(dv_1'\times dv_2')\\
&=\Pi_{i=1}^{j+1}C_i\dW\!\lc \WLh{j+1}{(\mX_{(\eps)}(G_1),\ell_{G_1})}(v_1),\WLh{j+1}{(\mX_{(\eps)}(G_2),\ell_{G_2})}(v_2)\rc.
\end{align*}
This concludes the proof.

\end{proof}

\subsection{Proof of \Cref{thm: universal}}\label{app:proof universal}
Recall the map $I_q:\lc \mathcal{G}(\R^d),d_{\mathcal{G},q}^{\scriptscriptstyle{(k)}}\rc\rightarrow \lc \mathcal{G}_q(\R^d),\dWLk\rc$. This map is continuous by the definition of $d_{\mathcal{G},q}^{\scriptscriptstyle{(k)}}$.
Hence $\mathcal{K}_q:=I_q(\mathcal{K})$ is compact. Then, we define the map $J_q:C(\mathcal{K}_q,\R)\rightarrow C(\mathcal{K},\R)$  sending $f_q$ to $f:=f_q\circ I_q$.
\begin{claim}\label{claim 1}
$J_q:C(\mathcal{K}_q,\R)\rightarrow C(\mathcal{K},\R)$ is a homeomorphism.  
\end{claim}

By \Cref{thm: universal MCNN pseudo}, we know that $\overline{\mathcal{N\!N}_k(\R^d)|_{\mathcal{K}_q}}=C(\mathcal{K}_q,\mathbb{R})$. Here $\mathcal{N\!N}_k(\R^d)$ denotes the collection of $k$-layer MCNNs (see \Cref{sec:MCNN}) and $\mathcal{N\!N}_k(\R^d)|_{\mathcal{K}_q}$ refers to the collection of restrictions of functions in $\mathcal{N\!N}_k(\R^d)$ to $\mathcal{K}_q$. By \Cref{eq:h H}, we have that $\mathcal{N\!N}_k^q(\R^d)|_\mathcal{K}=J_q(\mathcal{N\!N}_k(\R^d)|_{\mathcal{K}_q})$. 
Hence, by \Cref{claim 1} we have that
\[\overline{\mathcal{N\!N}_k^q(\R^d)|_\mathcal{K}}=J_q(\overline{\mathcal{N\!N}_k(\R^d)|_{\mathcal{K}_q}})=J_q(C(\mathcal{K}_q,\mathbb{R}))=C(\mathcal{K},\mathbb{R}).\]

\begin{proof}[Proof of \Cref{claim 1}]
We show that, in fact, $J_q$ is an isometry w.r.t. the sup metric on $C(\mathcal{K}_q,\R)$ and $C(\mathcal{K},\R)$. Given continuous $f_q,g_q:\mathcal{K}_q\rightarrow \R$, one has that
\begin{align*}
    &\sup_{(\mX_q(G),\ell_G)\in\mathcal{K}_q}|f_q((\mX_q(G),\ell_G))-g_q((\mX_q(G),\ell_G))|\\
    =&\sup_{(\mX_q(G),\ell_G)\in\mathcal{K}_q}|f((G,\ell_G))-g((G,\ell_G))|\\
    =&\sup_{(G,\ell_G)\in\mathcal{K}}|f((G,\ell_G))-g((G,\ell_G))|.
\end{align*}
Hence, $J_q$ is an isometric embedding. 
The inverse $(J_q)^{-1}$ of $J_q$ is identified as follows: for any $f\in C(\mathcal{K},\R)$, we let $f_q:\mathcal{K}_q\rightarrow\R$ be defined such that for any $(\mX_q(G),\ell_G)\in\mathcal{K}_q$, $f_q((\mX_q(G),\ell_G)):=f((G,\ell_G))$ ($f_q$ is well-defined: if $ (\mX_q(G_1),\ell_{G_1})=(\mX_q(G_2),\ell_{G_2})$, then $d_{\mathcal{G},q}^{\scriptscriptstyle{(k)}}((G_1,\ell_{G_1}),(G_2,\ell_{G_2}))=0$ and thus $ f((\mX_q(G_1),\ell_{G_1}))=f((\mX_q(G_2),\ell_{G_2}))$.). It is easy to check that $(J_q)^{-1}$ is the inverse of $J_q$ and is also an isometric embedding. This finishes the proof.
\end{proof}
\end{document}